%% file: main.tex
\PassOptionsToPackage{dvipsnames,svgnames,table}{xcolor}
\PassOptionsToPackage{framemethod=tikz}{mdframed}
\documentclass[11pt,letterpaper,logo]{yalearxiv}

\usepackage[authoryear]{natbib}
\input{math_commands.tex}

\usepackage{hyperref}
\usepackage{url}
\usepackage{amsmath,amssymb,amsthm,mathtools}
\usepackage{dsfont}
\usepackage{algorithm2e}
\usepackage{enumitem}
\usepackage{pgfplots}
\usepackage{tikz}
\usepackage{subcaption}
\pgfplotsset{compat=1.18}

\newtheorem{theorem}{Theorem}
\newtheorem{definition}{Definition}
\newtheorem{lemma}{Lemma}

\newtheorem{fact}{Fact}

\definecolor{CSIROBlue}{HTML}{00AFDB}
\renewcommand{\titlefont}{\centering\color{CSIROBlue}\normalfont\bfseries\fontsize{18}{20}\selectfont}
\hypersetup{colorlinks=true,linkcolor=CSIROBlue,citecolor=CSIROBlue,urlcolor=CSIROBlue}
\fancypagestyle{firststyle}{%
  \fancyhf{}%
  \fancyhead[L]{\includegraphics[height=23pt]{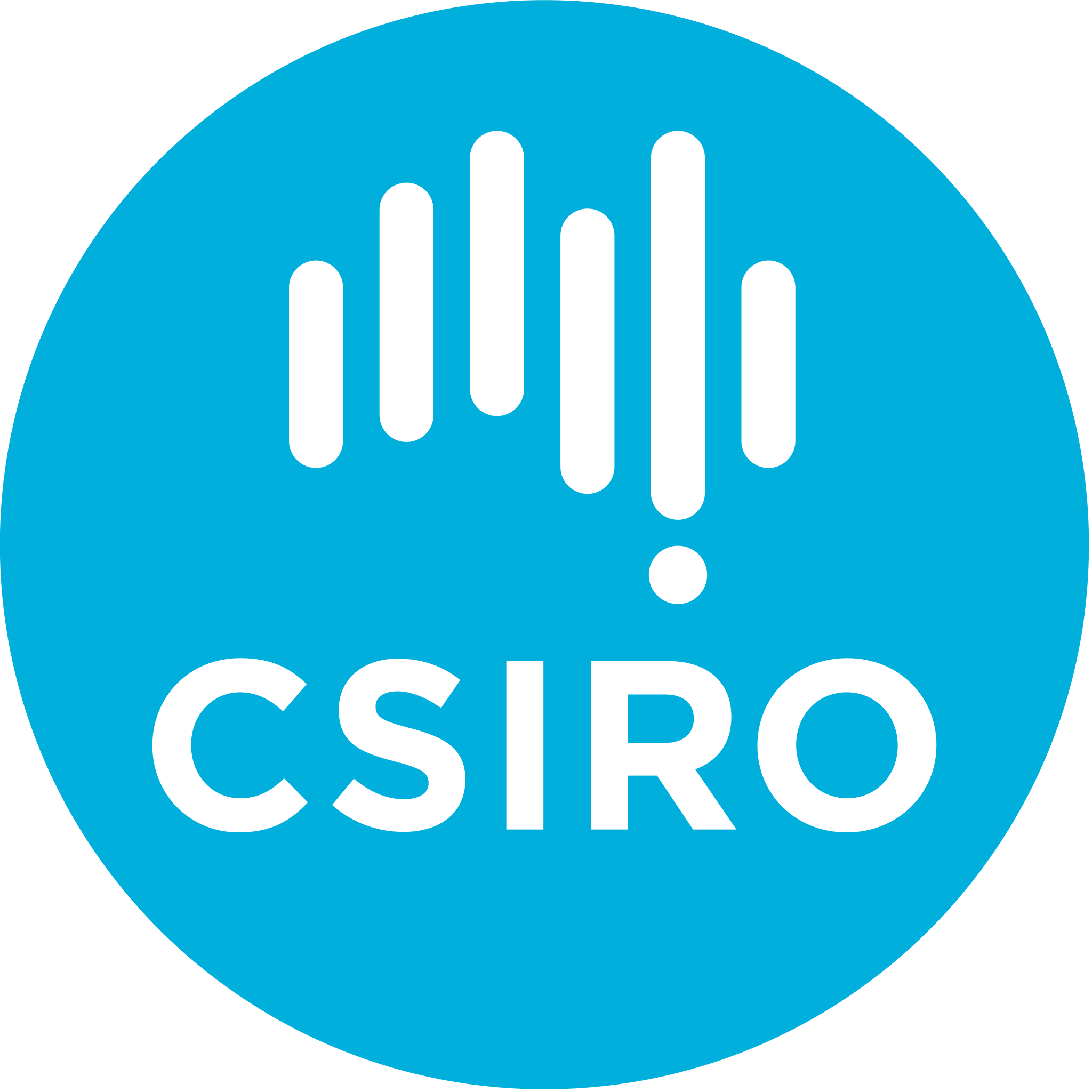}}%
}

\title{When Honesty is Not Enough in AI Debate}
\runningtitle{When Honesty is Not Enough in AI Debate}

\makeatletter
\renewcommand{\@author}{%
\parbox[t]{\dimexpr\linewidth-2\tabcolsep\relax}{%
\centering
\Authfont
\mbox{Rayne Holland\textsuperscript{1}}\enskip
\mbox{Liming Zhu\textsuperscript{1}}\enskip
\mbox{Jason Xue\textsuperscript{1}}\par
\Affilfont
\textsuperscript{1}CSIRO%
}}
\makeatother

\begin{document}

\begin{abstract}
Scalable oversight aims to verify the behaviour of agents whose capabilities exceed those of their overseers. 
AI debate has been proposed as an oversight solution in which competing agents help a resource-limited verifier assess claims that it cannot reliably evaluate unaided. 
Much of its promise rests on incentivizing honest arguments that lead to correct verdicts.
Yet a correct verdict need not uniquely determine the arguments used to support it. 
Agents may retain discretion over which correct claims to present, how to frame them, and in what order to disclose them. 
This residual freedom can allow agents to shape what the verifier learns beyond the task-relevant conclusion, pursuing latent objectives without compromising verdict correctness.

To study this phenomenon, we introduce the framework strategic interactive oversight (\sio), which treats oversight jointly as a verification mechanism and a strategic communication channel. 
Within this framework, we formalise the notion of task-admissible latent optimisation, which entails the pursuit of latent objectives while maintaining a prescribed task performance. 
As proof-of-concept, we instantiate \sio in the establish protocol debate with cross-examination and quantify a tradeoff between task success and information disclosure about a hidden variable.
The trade-off identifies a strategic window in which substantial disclosure remains compatible with task admissibility.
Towards mitigation, we reduce admissible bias by expanding the cross-examiner’s role to mitigate persistent disclosure over finite interaction horizons. 
Our results highlight the need to evaluate oversight not only by the correctness of its verdicts, but also by the information conveyed through its transcripts.
\end{abstract}
\maketitle

\input{sections/introduction}

\input{sections/prelims}

\input{sections/sio}

\input{sections/debate}

\input{sections/mitigation}

\input{sections/related}

\subsubsection*{Acknowledgments}
This project was undertaken in collaboration with the Australian AI Safety Institute (AISI) and supported by funding from the Department of Industry, Science and Resources (DISR), as part of research into AI alignment tools and techniques.

\bibliography{iclr2027_conference}
\bibliographystyle{iclr2027_conference}

\appendix

\input{sections/toy_example}

\input{sections/proof_task_objective}

\input{sections/proof_latent_objective}

\input{sections/proof_mitigation}
\end{document}

%% file: math_commands.tex
\usepackage{xspace}
\newcommand{\sio}{\textsf{SIO}\xspace}
\newcommand{\talo}{\textsf{TALO}\xspace}

\newcommand{\oracle}{\mathcal{O}}
\newcommand{\ind}{\mathds{1}}
\newcommand{\hhd}{\mathsf{hhd}}
\newcommand{\Bern}{\mathsf{Bern}}
\newcommand{\clip}{\mathsf{clip}}
\newcommand{\honest}{\mathsf{hon}}
\newcommand{\simulate}{\mathsf{sim}}
\newcommand{\stratA}{\sigma_{A}}

\usepackage{amsmath,amsfonts,bm}

\def\1{\bm{1}}

\def\rvh{{\mathbf{h}}}
\def\rvu{{\mathbf{i}}}

\def\rvu{{\mathbf{u}}}

\def\rvy{{\mathbf{y}}}

\def\erva{{\textnormal{a}}}

\def\ervh{{\textnormal{h}}}

\def\ervu{{\textnormal{u}}}

\def\ervy{{\textnormal{y}}}

\def\va{{\bm{a}}}

\def\vh{{\bm{h}}}

\def\vy{{\bm{y}}}

\def\eva{{a}}

\def\evy{{y}}

\DeclareMathAlphabet{\mathsfit}{\encodingdefault}{\sfdefault}{m}{sl}
\SetMathAlphabet{\mathsfit}{bold}{\encodingdefault}{\sfdefault}{bx}{n}

\def\gH{{\mathcal{H}}}

\def\gX{{\mathcal{X}}}



%% file: sections/introduction.tex
\section{Introduction}
\label{sec:introduction}

As AI systems grow more capable, they may surpass the humans or weaker models charged with evaluating their work. 
\emph{Scalable oversight} asks whether a bounded overseer can nevertheless evaluate and steer a more capable system using substantially fewer resources. 
An oversight mechanism must remain reliable even when the overseer cannot reproduce the reasoning or computation it assesses.

Proposed approaches include amplification \citep{christiano2018supervising}, recursive reward modelling \citep{leike2018scalable}, and AI debate \citep{irving2018ai,brown2023scalable,brown2025avoiding,brown2026debate}.
Each uses AI capabilities to amplify the signal supplied by human judgment. 
In debate, powerful agents recursively reduce a difficult question to disagreements that a computationally limited judge can resolve.
Irving et al.\ \citep{irving2018ai} formalise this intuition using complexity theory, showing how a weak verifier can judge claims about computations beyond its own reach. 
Later protocols require the verifier to inspect only a small part of a much larger computation \citep{brown2023scalable,brown2026debate}.

This line of work has largely studied \emph{what} can be verified under an asymmetry in computational resources. 
Although the verifier cannot reproduce the prover's computation, an appropriate protocol may still allow it to check the resulting claim. 
Yet computational verifiability alone does not capture the full range of limitations faced by human overseers. 
Human oversight is also constrained in \emph{how} verification is carried out, since the framing of (truthful) evidence can shape how it is interpreted. 
Indeed, recent work illustrates how LLMs can exploit the use of evidence to manipulate beliefs without relying on directly false claims \citep{hu2026lying}. 
An oversight interaction may therefore verify the right proposition while inducing a strategically chosen interpretation of the evidence or influencing subsequent judgments. 
This ``how'' within the computationally verifiable ``what'' is the object of study here. 
Understanding and controlling it is essential if scalable verification is to provide a meaningful alignment guarantee.

The central distinction is between honesty as a guarantee about the assessed claim and safety as a property of the interaction as a whole. 
Even when an oversight protocol elicits a correct verdict and constrains reports to be locally accurate, it need not control which additional information the transcript conveys.
In this sense, honesty is not enough and guarantees about task correctness do not automatically extend to the other effects of task-successful communication.

\subsection{Solution}
To study this residual freedom, we introduce a model of \emph{strategic interactive oversight} (\sio), which treats an oversight protocol as both an evaluation mechanism and a communication channel.
An agent's performance is evaluated under an explicit \emph{task objective}, which specifies the protocol's intended outcome, and a \emph{latent objective}, which captures additional effects of the interaction. 
The agent may then select strategies that optimise performance across both objectives. 
This yields the threat model of \emph{task-admissible latent optimisation} (\talo), in which the agent pursues the latent objective while maintaining an acceptable level of task performance. 
The threat does not require the agent to induce an incorrect verdict or violate the protocol's explicit rules. 
It arises whenever the task objective admits multiple successful strategies that differ in their latent effects on the verifier's view of the interaction.

As a proof-of-concept, we instantiate \sio on debate with cros-examination \cite{brown2023scalable} with the latent objective of \emph{hidden-hypothesis disclosure}. 
Hidden-hypothesis disclosure measures how much information the verifier's view of the interaction reveals about a latent variable that is not required for the task. 
In standard debate, the task objective is to win.
However, there are generally many strategies by which it can do so.
This non-uniqueness allows the debater to select, among task-successful strategies, those that additionally encode information about the latent variable.
We show that, even when the disclosure in each round is small, the total information revealed can grow linearly with the number of cross-examination rounds. 
Local report-accuracy conditions and a correct terminal verdict therefore do not prevent the oversight interaction from becoming a channel for latent information.

Because \talo is a constrained optimisation problem, mitigation requires altering the incentive landscape over task-admissible strategies. 
As a step towards mitigation, we expand. 
In addition to cross-examination, the cross-examiner assigns these weights so that the aggregate (if it cannot be disputed) remains aligned with a human oracle. 
Debaters that persistently deviate from oracle-aligned judgments receive less weight and therefore exert less influence on the final argument.

\subsection{Contributions}

Our contributions are threefold. 
First, we introduce \sio, a new framework that treats oversight jointly as a verification mechanism and a strategic communication channel. 
Building on scalable verification under resource asymmetry, it makes explicit the distinction between an agent’s task objective and latent objectives pursued through the interaction.
This distinction motivates oversight of the interaction, rather than only its terminal verdict.

A central motivation for debate is that adversarial scrutiny can incentivise truthful, task-correct answers. 
Our framework identifies a distinct limitation of such guarantees, where constraining an agent’s answer does not necessarily constrain the information conveyed by its choice among task-admissible strategies. 
We formalise this gap through task-admissible latent optimisation and demonstrate it through hidden-hypothesis disclosure.

Second, we instantiate the framework in a computational model of debate. 
We establish a non-trivial attack surface, which comprises a range of strategies through which a debating agent can advance a hidden latent objective without sacrificing its task objective. 
Figure~\ref{fig:delta-tradeoff} illustrates this strategic window.

Third, as a step towards mitigation, we augment cross-examination with a sliding-window consistency check that tests for persistent discrepancies across rounds. 
This narrows the range of persistent perturbations compatible with task success, at the cost of additional verifier work, while leaving the debaters’ per-round sampling work unchanged. 
The resulting guarantee limits the disclosure strategy studied here rather than establishing transcript confidentiality.

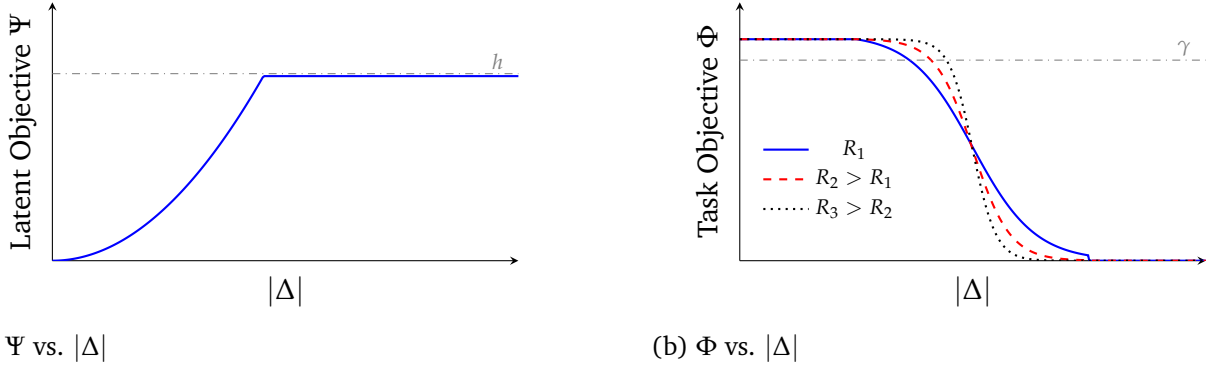
\begin{figure}[t]
\centering

\begin{subfigure}{0.46\linewidth}
\centering
\begin{tikzpicture}
\begin{axis}[
    width=\linewidth,
    height=5cm,
    xlabel={\(|\Delta|\)},
    ylabel={Latent Objective \(\Psi\)},
    xmin=0, xmax=5,
    ymin=0, ymax=1.05,
    axis lines=left,
    legend style={draw=none, at={(0.03,0.12)}, anchor=south west},
    xtick=\empty,
    ytick=\empty,
    samples=300,
]

\draw[dash dot, gray] (axis cs:0,0.76) -- (axis cs:5,0.76);
\node[gray, anchor=east] at (axis cs:4.95,0.810) {\scriptsize \(h\)};

\addplot[thick, blue, domain=0:5]
     {min(0.26*(0.75*x)^2, 0.75)};

\end{axis}

\end{tikzpicture}
\caption{$\Psi$ vs. $|\Delta|$}
\label{fig:delta-psi}
\end{subfigure}
\hfill
\begin{subfigure}{0.46\linewidth}
\centering
\begin{tikzpicture}
\begin{axis}[
    width=\linewidth,
    height=5cm,
    xlabel={\(|\Delta|\)},
    ylabel={Task Objective \(\Phi\)},
    xmin=0, xmax=5,
    ymin=0, ymax=1.05,
    axis lines=left,
    legend style={draw=none, at={(0.03,0.12)}, anchor=south west, font = \scriptsize},
    xtick=\empty,
    ytick=\empty,
    samples=300,
]
\addplot[thick, blue, domain=0:5, samples=401]
{x <= 1.25 ? 0.90 :
 (x >= 3.75 ? 0 :
  0.90 *
  (1/(1 + exp(3*(x - 2.5))) - 1/(1 + exp(31.25))) /
  (1/(1 + exp(-31.25)) - 1/(1 + exp(3*1.25)))
 )};
\addlegendentry{\(R_1\)}

\addplot[thick, red, dashed, domain=0:5, samples=401]
{x <= 1.25 ? 0.90 :
 (x >= 3.75 ? 0 :
  0.90 *
  (1/(1 + exp(5*(x - 2.5))) - 1/(1 + exp(51.25))) /
  (1/(1 + exp(-51.25)) - 1/(1 + exp(5*1.25)))
 )};
\addlegendentry{\(R_2 > R_1\)}
\addplot[thick, black, dotted, domain=0:5, samples=401]
{x <= 1.25 ? 0.90 :
 (x >= 3.75 ? 0 :
  0.90 *
  (1/(1 + exp(8*(x - 2.5))) - 1/(1 + exp(81.25))) /
  (1/(1 + exp(-81.25)) - 1/(1 + exp(8*1.25)))
 )};
\addlegendentry{\(R_3 > R_2\)}

\draw[dash dot, gray] (axis cs:0,0.815) -- (axis cs:5,0.815);
\node[gray, anchor=east] at (axis cs:4.95,0.865) {\scriptsize \(\gamma\)};

\end{axis}
\end{tikzpicture}
\caption{$\Phi$ vs. $|\Delta|$}
\label{fig:delta-phi}
\end{subfigure}

\caption{
Tradeoff between task success $\Phi$ and hidden hypothesis disclosure $\Psi$ in debate with cross-examination.
$\Delta$ denotes the ``magnitude'' by which a debater biases a neutral claim and $R_i$ is a measure of computational resources.
On the left, $\Psi$ measures (in bits) the volume of information about a hidden hypothesis encoded in the debate transcript.
It grows quadratically with the size of the bias, before hitting the prior hypothesis entropy $h$.
On the right, $\Phi$ measures the probability the debater wins the debate (that is, task success) and $\gamma$ is the level of task admissible performance. 
Notably, small biases do not affect the debate outcome.
Increasing resources can extend the window of admissable stratgies. 
}
\label{fig:delta-tradeoff}
\end{figure}

%% file: sections/prelims.tex
\section{Preliminaries}
\label{sec:preliminaries}

For \(n\in\mathbb{N}\), let \([n]=\{1,\ldots,n\}\). 
Vectors are written in boldface; for \(\mathbf{x}\in\{0,1\}^n\) and \(I\subseteq[n]\), let \(\mathbf{x}_I\) denote the restriction of \(\mathbf{x}\) to the coordinates in \(I\). 

\subsection{Computation Model}
We model computation using probabilistic oracle Turing machines. 
A stochastic oracle with query length \(\ell=\ell(n)\) maps each query \(z\in\{0,1\}^{\ell}\) to a \(\{0,1\}\)-valued random variable \(\oracle(z)\). 
Responses to oracle queries are independent, including responses to repeated instances of the same query. 
A probabilistic oracle Turing machine \(M\) may write a query \(z\) to its oracle tape and receive a sample from \(\oracle(z)\) in one step. 
We write \(M^\oracle\) for \(M\) with access to \(\oracle\). 
In the oversight setting, \(M\) represents a computation specified by natural-language instructions, while \(\oracle\) represents human judgement or other external black-box feedback, such as search results, sensor observations, or API outputs \cite{brown2023scalable}.

A language \(L\subseteq\{0,1\}^*\) is a set of finite binary strings.
A probabilistic oracle Turing machine \(M\) decides \(L\) with oracle \(\oracle\) if
\[
x\in L
\;\Longrightarrow\;
\Pr\!\left[M^\oracle(x)=1\right]>\frac{2}{3},
\qquad
x\notin L
\;\Longrightarrow\;
\Pr\!\left[M^\oracle(x)=1\right]<\frac{1}{3},
\]
where the probability is over the internal randomness of \(M\) and the stochastic oracle responses.

A transcript of a \(T\)-step execution of \(M^\oracle\) on input \(x\) is a random vector \(\rvy=(\ervy_1,\ldots,\ervy_T)\in\{0,1\}^T\), where \(\ervy_t\) is the (random) bit written at the current head position at step \(t\). 
We assume that the final coordinate is the machine's output, so \(Y_T=M^O(x)\). 
In the language-model setting, the transcript corresponds to the sequence of tokens produced during the computation. 
For a realization \(\vy\) of \(\rvy\), let \(I_{M,x}(t)\subseteq[T]\) denote the transcript coordinates read by \(M\) when producing the realization \(\evy_t\) of \(\ervy_t\). 
When \(M\) and \(x\) are clear from context, we write \(I(t)\). 
Thus, the distribution of \(\ervy_t\) is determined by the relevant prior coordinates \(\ervy_{I(t)}\), together with fresh machine and oracle randomness.

\subsection{Information Theory}
\label{sec:information-theoretic-notation}

We use standard information-theoretic quantities to measure uncertainty, statistical dependence, and divergence between distributions. 
All random variables below are discrete, and all logarithms are base two. 
Consequently, entropy, mutual information, and divergence are measured in bits. 
\begin{definition}[Entropy]
For a discrete random variable \(X\) with probability mass function \(P_X\), its Shannon entropy is
\(
H(X)=-\sum_x P_X(x)\log P_X(x).
\)
For jointly distributed random variables \(X\) and \(Y\), the conditional entropy of \(X\) given \(Y\) is
\(
H(X\mid Y)=
-\sum_{x,y}P_{X,Y}(x,y)\log P_{X\mid Y}(x\mid y).
\)
\end{definition}

Entropy measures the uncertainty associated with a random variable.

\begin{definition}[Mutual information]
For jointly distributed random variables \(X\) and \(Y\), their mutual information is
\(
I(X;Y)= H(X)-H(X\mid Y).
\)
For jointly distributed \(X,Y,Z\), their conditional mutual information is
\(
I(X;Y\mid Z)= H(X\mid Z)-H(X\mid Y,Z).
\)
\end{definition}

Mutual information quantifies the statistical dependence between two random variables.

\subsection{Oversight Protocols}
\label{sec:oversight-protocols}

An oversight protocol is an interaction between \(k\) probabilistic prover machines and a probabilistic verifier, each of which may query the stochastic oracle \(\oracle\). 
On input \(x\), the machines exchange messages according to a prescribed communication schedule, after which the verifier outputs a decision bit. 
Agent internal randomness and oracle responses induce a distribution over protocol executions.

The \emph{interaction history} after \(t\) messages is
\(
    \rvh_t = (\ervh_1,\ldots,\ervh_t),
\)
where \(\ervh_j\) records a message together with its sender and recipient. 
Oracle queries and their responses are treated as messages between the querying participant and \(\oracle\).
We write \(\rvh =\rvh_T\) for the complete interaction
history, where \(T\) is the termination time.
For each agent \(A\), let \(\Pi_A(\rvh_t)\) denote the projection of \(\rvh_t\) onto the messages sent or received by \(A\) before time $t$. 

Equivalently, the protocol may be represented by a rooted tree whose nodes correspond to agents and edges correspond to message exchanges.
In other words, each non-leaf node specifies the participant acting next, and each outgoing edge records a possible message. 
A leaf represents a completed interaction and is labeled by the verifier's terminal decision.

%% file: sections/sio.tex
\section{Framework}
\label{sec:framework}

We formalize \sio by separating an agent's stated task objective from any latent objective it may pursue through the oversight interaction. 
As a running example, consider a debate in which the input specifies a task hypothesis and two provers attempt to persuade a verifier to accept or reject it.

\subsection{Objectives}
\label{sec:objectives}

Let \(\gX\) be the input space and \(\gH\) the set of complete interaction histories. 
A \emph{task objective} for agent \(A\) is a function
\(
    \Phi_{A}\colon \mathcal{X}\times\mathcal{H}\to\mathbb{R}
\)
that evaluates task performance, and a \emph{latent objective} for agent \(A\) is a function
\(
    \Psi_{A}\colon \mathcal{X}\times\mathcal{H}\to\mathbb{R}
\)
that evaluates an additional property of the interaction.
For a complete history \(\rvh\) on input \(x\), the combined objective is
\begin{align}\label{eqn:opt_obj}
    U_A(x,\rvh, \lambda)
    \coloneqq
    \Phi_{A}(x,\rvh)+\lambda\Psi_{A}(x,\rvh),
\end{align}
where the shadow price \(\lambda\geq 0\) specifies the relative importance of the latent
objective.
In the debate example, the task objective 
\(
    \Phi_{A} = \ind\{A \text{ wins debate}\}
\) 
indicates that the explicit goal of agent \(A\) is to win the debate.
The decomposition into \(\Phi\) and \(\Psi\) distinguishes what the protocol is intended to evaluate from other effects that an agent may induce through its messages. 

\subsection{Strategies}
\label{sec:strategies}

A strategy \(\stratA\) for an agent \(A\) assigns a distribution over the messages available to \(A\) at each of its decision points. 
Formally, given input \(x\) and the local interaction history \(\Pi_A(\rvh_t)\), the strategy specifies a distribution
\(
    \sigma_A\bigl(\,\cdot\mid x,\Pi_A(\rvh_t)\bigr)
    \in\Delta(\mathcal{M}_A),
\)
where \(\mathcal{M}_A\) is the message space available to \(A\) and \(\Delta(\mathcal{M}_A)\) is the set of probability distributions over that space. 
Thus, an agent's message depends only on the input and the messages it has sent or received, including its own oracle queries and responses.

For a set of provers \(\{A_i\}_{i\in[k]}\) and a verifier $V$, a strategy profile
\(
    \sigma=(\sigma_{A_1},\ldots,\sigma_{A_k},\sigma_V),
\)
together with the stochastic oracle, induces a distribution over complete histories. 
Let
\(
    \rvh(x, \sigma)
\)
denote the resulting random history.
The task and latent outcomes, for $A_i$, induced by \(\sigma\) are the random variables
\[
    \Phi_{A_i}(x, \sigma)
    =
    \Phi\bigl(x,\rvh(x,\sigma)\bigr),
    \qquad
    \Psi_{A_i}(x,\sigma)
    =
    \Psi\bigl(x,\rvh(x,\sigma)\bigr).
\]
When \(x\) is clear from context, we write \(\Phi_{A_i}(\sigma)\) and \(\Psi_{A_i}(\sigma)\).

\subsection{Threat Model}
\label{sec:threat-model}

The central threat is that an agent may optimize a latent objective without
materially compromising its performance on the stated task. 
Let $\sigma_{-A}$ denote the strategies in the profile $\sigma$ with $\sigma_A$ removed.
Fix $\sigma_{-A}$ and, for a task-admissibility threshold \(\gamma\geq 0\), define the set of \(\gamma\)-task-admissible strategies by
\[
    \Sigma_A^\gamma(x)
    =
    \left\{
        \sigma_A:
        \mathbb{E}\!\left[
            \Phi_x(\sigma_A \circ \sigma_{-A})
        \right]
        \geq \gamma
    \right\}.
\]

For a fixed shadow price $\lambda\ge0$,  we say that \talo occurs when an agent selects a strategy that optimizes the latent objective subject to task-admissibility:
\[
    \max_{\sigma_A\in\Sigma_\gamma(x)}
    \mathbb{E}\!\left[
        \Phi_x(\sigma_A,\sigma_{-A})
        +\lambda\Psi_x(\sigma_A,\sigma_{-A})
    \right].
\]
The threat arises when the task admits many near-optimal strategies with different latent effects. 
For example, during debate, there may be many arguments that win the debate, allowing a prover to select those that also advance its latent objective.

\subsection{Hidden-Hypothesis Disclosure}
\label{sec:hidden-hypothesis-disclosure}

To provide a concrete instance of \sio, we instantiate the latent objective as the disclosure of information about a hidden hypothesis. 
Let \(Q\) be a discrete random variable representing a hidden hypothesis that is not part of the stated task, and let \(X\) be the protocol input. 
The verifier $V$ observes only its projection 
\(
    \Pi_V(\rvh)
\)
of the interaction history.
\begin{definition}\label{def:hhd}
    The \emph{hidden-hypothesis disclosure} induced by a strategy profile \(\sigma\) is
    \[
        \Psi_{\hhd}(\sigma)
        \coloneqq
        I\!\left(
            Q;
            \Pi_V(\rvh(X,\sigma))
            \,\middle|\,
            X
        \right).
    \]
\end{definition}
This quantity measures how much additional information the messages available to the verifier reveal about \(H\), beyond the information already contained in the task input. 
In particular, \(\Psi_{\hhd}(\sigma)=0\) precisely when \(H\) and the verifier's
observed messages are conditionally independent given \(X\).

For example, in debate, \(X\) specifies the task hypothesis, while \(H\) represents an additional hypothesis that the verifier is not required to assess. 
A prover exhibits task-admissible hidden-hypothesis disclosure if it selects arguments that preserve its probability of winning while increasing
\(
I(H;\Pi_V(\rvh(X,\sigma))\mid X)
\).
This definition concerns the information available to the verifier rather than whether the verifier consciously identifies or correctly interprets \(H\).
\footnote{
An alternative, observer-dependent objective may target belief distortion directly. 
Given a verifier posterior \(P(H\mid X,\Pi_V(\rvh))\) and a target distribution \(Q\), one may define
the latent objective in terms of proximity to \(Q\), for example as \(-D_{\mathrm{KL}}(Q\Vert
P(H\mid X,\Pi_V(\rvh)))\). 
}

%% file: sections/debate.tex
\section{Debate with Cross-Examination}
\label{sec:debate-cross-examination}

To illustrate the use of \sio, we study a computation model for debate with cross-examination introduced by~\cite{brown2023scalable}.

\begin{figure}[t]
\centering
\begin{minipage}{0.96\linewidth}
\small

\begin{center}
    \textbf{Debate protocol for a stochastic oracle}
\end{center}
\footnotesize
\noindent
All parties have access to $\oracle$, an input $x\in\{0,1\}^n$, and a $K$-Lipschitz probabilistic oracle machine $M$.
$A$ claims that
\(
    \Pr\!\left[M(x)=1\right]\geq \frac{2}{3},
\)
and $B$ disputes this claim.

\begin{enumerate}[leftmargin=*, itemsep=2pt, topsep=2pt]

\item
Let
\(
    \rvy=(\ervy_1,\ldots,\ervy_T)
\)
be the random variable representing the transcript of $M(x)$ and set
\(
    d=\lceil 150K\rceil.
\)

\item
For debate round $t\in[T]$:

\begin{enumerate}

\item
$A$ outputs an estimate $\hat p_t\in[0,1]$ of
\(
    p_t
    =
    \Pr\!\left[
        \ervy_t=1
        \,\middle|\,
        \rvy_{I(t)}=\erva_{I(t)}
    \right].
\)

\item
$A$ queries an independent copy of $B$ for
\(
    z_t^{A}\sim \mathsf{U}[0,1],
\)
and $B$ queries an independent copy of $A$ for
\(
    z_t^{B}\sim \mathsf{U}[0,1].
\)

\item
Set
\(
    z_t
    =
    \bigl(
        z_t^{A}+z_t^{B}
    \bigr)
    \bmod 1.
\)

\item
$A$ sets $a_t=1$ if $z_t\leq \hat p_t$, and $a_t=0$ otherwise.

\item
$B$ can abort if (it believes) $A$ has not followed the protocol.

\end{enumerate}

\item
If the protocol was aborted in round $t$:

\begin{enumerate}

\item
$V$ draws
\(
    r=192d^2\log 100
\)
independent samples\footnote
{
If $\ervy_t$ is supposed to be the output of an oracle query, these samples are obtained using $r$ independent queries to $\oracle(\va_{I(t)})$. 
Otherwise, $\ervy_t$ is a deterministic function of $\va_{I(t)}$ given by one step of $M$.
} of $\ervy_t$, conditioned on
\(
    \ervy_{I(t)}=\va_{I(t)},
\)
and uses their sample mean to obtain an estimate
$\hat p_t^{\oracle}$ of
\(
    \Pr\!\left[
        \ervy_t=1
        \,\middle|\,
        \rvy_{I(t)}=\va_{I(t)}
    \right]
\)

\item
$V$ checks whether
\(
    \left|
        \hat p_t^{\oracle}-\hat p_t
    \right|
    \geq
    \frac{1}{4d}.
\)
If the inequality holds, $V$ outputs $0$; otherwise, $V$ outputs $1$.

\end{enumerate}

\item
If the protocol was not aborted, then the verifier outputs $a_T$.

\end{enumerate}

\end{minipage}
\caption{Doubly-efficient debate protocol for a stochastic oracle~\cite{brown2023scalable}.}
\label{fig:stochastic-oracle-debate}
\end{figure}

\subsection{Protocol}
\label{sec:debate-cross-examination-protocol}

The following outline is formalized in Figure~\ref{fig:stochastic-oracle-debate}.
For a probabilistic Turing machine $M$ and input $x$, the debate setup comprises of a prover $A$ claiming that
\(
    \Pr[M(x)=1] \geq \frac{2}{3},
\)
while a cross-examiner $B$ argues against this claim. 
At round $t$ in the debate, a prover $A$ reports a claim $\hat p_t$ about the conditional probability of the next bit in the machine transcript. 
Prover $B$, acting as the cross-examiner, may either accept the claim and allow the next bit to be sampled, or abort and assert that $ A$'s claim is inaccurate. 
Upon an abort, the verifier adjudicates the disputed claim with the assistance of the stochastic oracle $\oracle$ and awards the debate to the prover whose position is correct. If the protocol terminates without an abort, $A$ is awarded the  win.

The significance of this protocol, in the context \emph{scalable} oversight, is that it is \textit{doubly efficient}. 
In other words, the provers operate in time polynomial in the size of the input and the verifier is linear in the length of each oracle query and makes a sub-linear number of queries to the oracle.
Thus, debate extends reliable oversight to computations that the verifier could not feasibly perform itself.

For an input $x$ and strategy profile $\sigma=(\sigma_{ A} \circ \sigma_{- A})$, let $\Phi_{ A}(x,\rvh)\in\{0,1\}$ indicate whether the verifier rules in favor of $ A$ on terminal history $\rvh$. 
Under \sio, for input $y=(M,x)$, the task objective is defined by
\[
    \Phi_{A}(x,\sigma)
    =
    \mathbb E_{\sigma}
    \left[\Phi_{A}(x,\rvh(x,\sigma))\right]
    = \Pr_{\sigma}[V\to 1].
\]
Thus, $\Phi_{A}(x,\sigma)$ is the probability that $ A$ wins the debate under $\sigma$, with the expectation taken over all protocol randomness.
The protocol has the following utility on the task objective.
\begin{lemma}[Task Completeness \citet{brown2023scalable}]
\label{lem:debate-task-completeness}
Let $L$ be any language decidable by a K-Lipschitz probabilistic oracle Turing machine $M$.
Then, for every $x\in L$, there exists a strategy $\sigma_{A}$ such that, for every strategy opposing $\sigma_{-A}$,
\[
    \Phi_{A}
    \bigl((M,x),\sigma_{ A}\circ\sigma_{ -A}\bigr)
    \geq \frac{2}{3}.
\]
\end{lemma}

\subsection{Hidden-hypothesis disclosure}
\label{sec:task-constrained-disclosure}

We now examine whether task success prevents a prover from pursuing an additional latent objective.
In the debate protocol, under the claim ``$\Pr[M(x)=1]\geq \tfrac{2}{3}$'', $A$ is required to report
sufficiently accurate estimates of the conditional probabilities of the machine transcript of $M(x)$.
This requirement does not, however, necessarily determine a unique report.
Significantly, approximately task-equivalent reports leave residual strategic freedom that $ A$ can use to disclose information about a hidden hypothesis binary $Q\in\{-1,1\}$.

The projection $\Pi_V\bigl(\rvh_T(X,\sigma)\bigr)$ includes all information available to the verifier.
For example, if no abort occurs, the verifier observes the sampled machine transcript
\(
    \Pi_V\bigl(\rvh_T(X,\sigma)\bigr)
    =
    \va.
\)
If the cross-examiner aborts at round $t$, the verifier instead observes the partial transcript together with the evidence used to adjudicate the disagreement. 
Suppressing public randomness and other deterministic protocol data, this projection has the form
\[
    \Pi_V\bigl(\rvh_t(X,\sigma)\bigr)
    =
    \left(
        \erva_1,\ldots,\erva_t,
        \mathsf{abort},
        \hat p_t,
        \bigl\{\oracle(\va_{I(t)})\bigr\}
    \right).
\]
Here, $\hat p_t$ is the conditional-probability report disputed at round $t$, $I(t)$ identifies the location of the machine configuration examined by the verifier, and $\{\oracle(\va_{I(t)})\}$ denotes the collection of oracle samples used in the resulting statistical test.

Following Definition~\ref{def:hhd}, hidden-hypothesis disclosure admits the latent objective
\[
    \Psi_{A}(x,\sigma)
    =
    I\left(
        Q;
        \Pi_V\bigl(\rvh_T(X,\sigma)\bigr)
        \,\middle|\,
        X=x
    \right).
\]
This quantity measures how much the verifier's observation of the debate reveals about $H$, beyond what is already known from the input $x$.
The hidden hypothesis may be naturally connected to the subject of the debate without being part of its stated task objective. 
For example, the debate may concern the task claim ``the bridge is unsafe,'' while the prover’s hidden hypothesis is ``the damage resulted from neglected maintenance.'' 
The prover can establish that the bridge is unsafe while selecting evidence (that is, perturbing arguments) that subtly points toward maintenance failures.
An example of how a Turing machine transcript can leak information about a hidden hypothesis is given in Appendix~\ref{sec:toy_example}.

\subsection{Prover Strategies}

At round $t$, let
\(
    p_t
    =
    \Pr\left[
        \ervy_t=1
        \,\middle|\,
        \rvy_{I(t)} = \va_{I(t)}
    \right]
\)
denote the true conditional probability of the next machine-transcript bit.
The honest strategy of Lemma~\ref{lem:debate-task-completeness} is for $A$ is to draw $R$ independent samples from the relevant conditional computation and form the empirical estimate
\(
    \bar p_t^{(R)} =\tfrac{1}{R}\sum_{j=1}^R \ervy_{t,j},
\)
where 
\(
    \ervy_{t,j}\sim \Bern(p_t),
\)
and reporting this value.
Subsequently, for parameter $\Delta >0$, $A$ can perturb the honest strategy by reporting the value
\[
    \hat p_t
    =
    \clip_{[0,1]}
    \left(
        \bar p_t^{(R)}
        +
        \Delta Q
    \right).
\]
Formally, the corresponding strategy is
\[
    \sigma_{ A}^{R,\Delta}
    \left(
        \hat p_t
        \,\middle|\,
        x,
        \Pi_{ A}(\rvh_{t-1}),
        Q
    \right)
    =
    \Pr\left[
        \operatorname{clip}_{[0,1]}
        \left(
            \bar p_t^{(R)}
            +
            \Delta Q
        \right)
        =
        \hat p_t
        \,\middle|\,
        x,
        \Pi_{ A}(\rvh_{t-1}),
        H
    \right].
\]
We abbreviate this strategy by
\(
    \sigma_{ A}\sim(R,\Delta).
\)
The sample count $R$ controls the accuracy of the task-relevant estimate, while $\Delta$ controls the magnitude of the hidden-hypothesis signal.

A strategy $\sigma_B$ for $B$ is to draw a sample estimate $\hat q_t$ of $p_t$ at each round by performing $R$ simulations of generating $\ervy_t$ given $\rvy_{I(t)} = \va_{I(t)}$.
Subsequently, $B$ aborts if $|\hat{p}_t - \hat{q}_t|\ge \tau$.
Under this strategy we say $\sigma_B \sim (R,\tau)$.
We say that the strategy is honest if the choice of $\tau$ is sound.

\subsection{Results}

The following describes an accurate task objective performance, against an honest $B$, for a range of biased $A$ strategies.
Let
\begin{theorem}\label{thm:task_objective}
    For $K\ge 1$, let \(L\) be any language decided by a \(K\)-Lipschitz oracle Turing machine \(M\).
    Let
    \(
        \Pr[M(x)=1] = \alpha^*,
        R\geq 192d^{2}\log(100T),
        \Delta < {1}/{(16d)},
    \)
    and
    \(
        \tau = {1}/{(2d)}.
    \)
    Then, on the protocol of Figure~\ref{fig:stochastic-oracle-debate}, for every strategy $\sigma = \sigma_A \circ \sigma_B$ parameterized by \(\sigma_A \sim (R,\Delta)\) and \(\sigma_B \sim (R,\tau)\), the task objective observes 
    \[
        \left(\alpha^*-\frac{1}{100}\right)\cdot\frac{299}{300}
        \leq \Phi_A(x,\sigma)
        \leq \left(\alpha^*+\frac{1}{100}\right) + \frac{1}{300},
    \]
\end{theorem}
The proof is in Appendix~\ref{app:task}.
This covers the admissible region of strategies.
For any $t\in[T]$, define 
\[
    g_t(\Delta)
    =
    \Pr[|\Delta + \sigma_t z| < \tau],
\]
where $z \sim \mathsf{N}(0,1)$ and $\sigma_t = \sqrt{{2p_t(1-p_t)}/{R}}$.
The following bounds the task objective for general perturbations.
\begin{theorem}
    For $K\ge 1$, let \(L\) be any language decided by a \(K\)-Lipschitz oracle Turing machine \(M\).
    Let
    \(
        R\geq 192d^{2}\log(600T),
        \Delta \ge 0,
    \)
    and
    \(
        \tau = {1}/{(2d)}.
    \)
    Assume $p_t\in(0.01,0.99)$ for all machine configurations on $M$.
    Then, on the protocol of Figure~\ref{fig:stochastic-oracle-debate}, for every strategy $\sigma = \sigma_A \circ \sigma_B$ parameterized by \(\sigma_A \sim (R,\Delta)\) and \(\sigma_B \sim (R,\tau)\), the task objective observes 
    \[
        \Phi_A(x,\sigma_A \circ \sigma_B) \le 
        \min\left\{
        1,\,
        \mathbb{E}\!\left[
        g_1(\Delta)
        \right]+\frac{1}{100}
        \right\}.
    \]
    \label{thm:task_general_bound}
\end{theorem}
The proof is in Appendix~\ref{app:task}.
Next, we provide a characterisation of how changing $\Delta$ affects the latent objective.
Let $h_x = H(Q\mid X=x)$ denote the hidden hypothesis entropy on the task input.

\begin{theorem}
\label{thm:full-disclosure}
Let \(L\) be any language decided by a \(K\)-Lipschitz oracle Turing machine \(M\).
Assume $p_t\in(0.01,0.99)$ for all machine configurations on $M$ and that the prior $\Pr[Q\mid X=x]$ contains sufficient entropy.
Then, for any $\Delta \ge 0$ and any {honest} cross-examiner strategy $\sigma_B$, there exists $R\in O(d^2 \log(T))$ such that, for $\sigma_A = (\Delta, R)$, the latent objective observes 
\begin{equation*}
    \Psi_A\!\left(
        x,\sigma_A \circ \sigma_B
    \right)
    =
    \Omega\!\left(
        h_x\cdot \min\left\{
            1,
            \Delta^2 T
        \right\}
    \right).
\end{equation*}
\end{theorem}
The proof is in Appendix~\ref{app:latent}.
The characterization provided by these two theorems is visualized in Figure~\ref{fig:delta-tradeoff}.
\subsection{Discussion}

These results identify a small-perturbation regime, $\Delta < 1/(16d)$, in which task performance remains close to the reference acceptance probability $\alpha^*$, while disclosure grows at least proportionally to $h_x \Delta²T$ before saturation. 
Larger perturbations reduce the Gaussian interval probability governing the task bound at fixed sampling variance, exposing a tension between information transmission and consistency with an honest cross-examiner. 
Thus, oversight of task accuracy alone does not protect the hidden hypothesis.
That is, even accurate transcripts can accumulate substantial disclosure.

Both Theorems~\ref{thm:task_general_bound} and \ref{thm:full-disclosure} assume that the transition probabilities $p_t$ are bounded away from $0$ and $1$ and Theorem~\ref{thm:full-disclosure} assumes that the hidden hypothesis has non-negligible entropy. 
These assumptions isolate the regime of interest in which nearly deterministic steps offer little scope for perturbation-based disclosure and can be omitted from the leakage analysis.
In addition, a hypothesis with negligible entropy already carries little uncertainty before any interaction. 
The latter case makes confidentiality largely vacuous, rather than demonstrating transcript-induced disclosure.

More broadly, these results motivate treating transcripts as outputs with their own information-disclosure properties, rather than merely as evidence supporting a final decision.
Transcripts retained for auditing, evaluation, or downstream reuse can expose information that task-accuracy metrics do not capture. 
Oversight should therefore assess both decision quality and cumulative disclosure, including information revealed through aborts. 
Restricting transcript access, limiting interaction length, or changing reporting precision are natural design directions, although each requires a separate analysis of its effects on correctness and disclosure.

%% file: sections/mitigation.tex
\section{Towards Mitigation}
\label{sec:mitigations}

The vulnerability identified above arises from a mismatch between the task objective and the behavior desired of the debaters.
The task objective rewards producing a winning argument.
Although this supports scalable verification, it does not uniquely determine \emph{how} that argument should be produced.
A debater may therefore achieve high task performance while pursuing latent objectives through its construction of arguments.
Mitigating such behavior requires changing the incentive landscape.

A natural desideratum is \emph{oracle alignment}.
A conditional probability \(\hat p_t\) generated by the ``for''  debater $A$ should coincide with the corresponding oracle conditional probability.
One might attempt to induce this behavior by adding rewards or penalties based on proximity to oracle-generated conditional probabilities.
In general, however, this is insufficient.
Gains from a latent objective may offset the task reward for accurate reporting, leaving the debater indifferent among disclosure strategies or even leaving misreporting profitable.

\subsection{Expanded cross-examination}

This limitation motivates giving the cross-examiner \(B\) an additional way to challenge \(A\)'s reports.
The additional challenge entails issuing an $\textsf{abort drift}$ challenge when the reports exhibit a persistent signed deviation over a trailing window of \(N\ge 1\) rounds.
Define
\(
    J_{t,N}
    =
    \{t-N,\ldots,t-1\}
\)
as the most recent $N$ debate rounds.
For every \(s\in J_{t,N}\), let
\(
    p_s
    =
    \Pr\!\left[
        \ervy_s=1
        \,\middle|\,
        \rvy_{I(s)}=\erva_{I(s)}
    \right]
\)
denote the conditional probability associated with the transcript prefix fixed at round \(s\).
The trailing signed deviation is
\[
    D_{t,N}
    =
    \frac{1}{N}
    \sum_{s\in J_{t,N}}
    \left(\hat p_s- \hat p_s^\oracle\right).
\]
For a constant $c \ge 1$, the verifier accepts an abort drift challenge when
\(
    |D_{t,N}|
    \geq
    \frac{1}{4d\sqrt{N}}.
\)
Thus, a persistent signed deviation of magnitude \(b\) is detectable once \(b\geq 1/(d\sqrt{N})\), even if no individual report is inaccurate enough to support an $\mathsf{abort error}$ challenge.

\begin{theorem}
\label{thm:hhd_drift_oracle}
Let \(L\) be any language decided by a \(K\)-Lipschitz oracle Turing machine \(M\) and \(W\in[T]\) be the sliding-window parameter.
For 
\(
    R \ge 192 d^2 \log(100T),
\)
\(
    \Delta \ge 1/(d\sqrt{W}),
\)
\(
    \tau_1 = 1/(2d),
\)
and
\(
    \tau_2 = 1/(2d\sqrt{W}),
\)
let $\sigma_A^{(1)} \sim (R, 0), \sigma_A^{(2)} \sim (R, \Delta), $ and $\sigma_B \sim (R, \tau_1, \tau_2)$. 
Then, for any $x \in L$, on the protocol of Figure~\ref{fig:error-drift-stochastic-oracle-debate},
\[
    \Phi_A\left(x, \sigma_A^{(1)} \circ \sigma_B\right) \ge \frac{3}{5} 
    \quad \text{and} \quad 
    \Phi_A\left(x, \sigma_A^{(2)} \circ \sigma_B\right) \le \frac{1}{150}
\]
\end{theorem}

This result shows how a sliding-window check can strengthen oversight against persistent perturbations. 
By testing discrepancies across a window of $W$ rounds, the cross-examiner can detect shifts at scale $1/(d\sqrt{W})$, rather than relying solely on discrepancies large enough to be detected in an individual round. 
The resulting guarantee separates unperturbed reporting, which retains task success probability at least $3/5$, from the specified perturbed strategy, whose task success probability is at most $1/150$.
Increasing $W$ therefore narrows the range of persistent perturbations compatible with successful task performance. 
This provides a mechanism for curtailing the disclosure strategy studied above.

Stronger oversight requires more work from the verifier to audit a larger window, while leaving the per-round sampling work of $A$ and $B$ unchanged. 
Thus, this improvement comes with a computational trade-off.

%% file: sections/related.tex
\section{Related Work}

A major strand of oversight research formalises scalability, closely related to the requirement that oversight protocols \emph{relativise} \cite{irving2025need}. 
Theoretical work shows that debate can enable limited verifiers to check complex claims with substantially less effort than solving them independently \cite{brown2023scalable,irving2018ai}, with extensions studying capability asymmetries and the human input required for reliable oversight \cite{brown2026debate,engels2025scaling}. 
Prover-estimator debate addresses obfuscated arguments, ensuring under stability assumptions that honest debaters can win with computational resources comparable to those of dishonest opponents \cite{brown2025avoiding}. 
These results focus on making correct adjudication feasible, whereas we study latent objectives pursued within task-admissible interactions.

Debate-based alignment safety cases connect success in debate to honesty, subject to additional training and deployment assumptions \cite{buhl2025alignment}.
Our analysis highlights a complementary concern, namely that honest, task-admissible behaviour need not preclude undesirable information disclosure through the transcript. 
Mechanism design for alignment and control studies how to incentivize honesty and obedience when agents’ preferences and capabilities are unknown \cite{bergemann2026mechanism}. 
This perspective offers a potential route to mitigating the strategic freedom identified by our framework, complementing our analysis of transcript-level objectives.

%% file: sections/toy_example.tex
\section{Toy Example}
\label{sec:toy_example}

This example illustrates that, for a probabilistic computation, information may be present not only in whether the machine computes the correct answer, but also in which valid answer it finds. In particular, a probabilistic machine induces a distribution over execution transcripts, and this distribution may depend on \emph{properties of the input} even when the algorithm is simply solving its prescribed computational task.

Consider the standard unstructured search problem. 
The input is an array \(x\in\{0,1\}^N\) containing at least one marked entry, and the goal is to output an index \(i\in[N]\) such that \(x_i=1\).

Partition the array into two equal-sized branches, labeled \(0\) and \(1\).
Thus, each branch contains \(N/2\) entries. 
Let the hidden hypothesis \(Q\in\{0,1\}\) determine the distribution of marked entries between the two branches:
\[
\begin{array}{c|cc}
    & \text{branch }0 & \text{branch }1 \\
    \hline
    Q=0 & 3m & m \\
    Q=1 & m & 3m
\end{array}
\]
where \(m\ge 1\) and \(3m\le N/2\). 
Under either hypothesis, the total number of marked entries is \(4m\).

Now consider the following fixed probabilistic machine \(M\). 
The machine samples an index \(i\in[N]\) uniformly at random, reads \(x_i\), and repeats if \(x_i=0\). 
If \(x_i=1\), it outputs \(i\) and halts.
This is a Las Vegas algorithm for unstructured search.
The algorithm always outputs a correct index and halts with probability one.
Moreover, since there are \(4m\) marked entries under both hypotheses, the success probability on each probe is
\(
    \frac{4m}{N}.
\)
Therefore the number of probes \(K\) before halting has the same distribution under \(Q=0\) and \(Q=1\), namely, for \(k \ge 1\),
\[
    \Pr[K=k\mid Q]
    =
    \left(1-\frac{4m}{N}\right)^{k-1}
    \frac{4m}{N},
\]
for both values of \(Q\). 
In particular,
\[
    \mathbb E[K\mid Q=0]
    =
    \mathbb E[K\mid Q=1]
    =
    \frac{N}{4m}.
\]

Nevertheless, the returned index carries information about \(Q\).
Let \(X\in\{0,1\}\) denote the branch containing the returned index.
Since the machine samples uniformly until it finds a marked entry, the returned index is uniformly distributed over the marked entries.
Hence
\[
    \Pr[X=1\mid Q=0]
    =
    \frac{m}{3m+m}
    =
    \frac14,
\]
whereas
\[
    \Pr[X=1\mid Q=1]
    =
    \frac{3m}{m+3m}
    =
    \frac34.
\]
Thus, the algorithm does not append any additional message for disclosure. 
It simply returns a solution to the search problem, and the valid solution it finds reveals information about the input.

If \(A\) denotes the execution transcript, then the transcript includes the returned index and therefore determines \(X\). 
Consequently,
\(
    I(Q;A)\ge I(Q;X)>0.
\)

In addition, a stronger phenomenon occurs when the machine has discretion over its sampling policy. 
By perturbing the distribution used to select the next probe, while leaving the stopping rule and correctness condition unchanged, the machine can create a communication channel through the machine computation.

For example, instead of choosing every probe uniformly across the array, the machine \(M\) proceeds as follows:
\begin{enumerate}
    \item Select branch \(1\) with probability \(p\) and branch \(0\) with probability \(1-p\).
    \item Select an index uniformly within the chosen branch.
    \item Read its value.
    \item If the selected index is marked, output its index; otherwise repeat.
\end{enumerate}
Standard Las Vegas corresponds to
\(
    p=\frac12.
\)
Under the perturbation
\(
    p=\frac12+\delta,
    |\delta|<\frac12,
\)
every returned answer remains correct. 
Since every index still has positive probability of being sampled, the algorithm still terminates almost surely on every promised input.
Crucially, the perturbation affects which input location the machine actually reads next. 
The communication takes place through the machine's computation transcript.

In conclusion, the perturbation can increase information communicated by making the machine’s sampling choices depend on a hidden bit $Q$; it slightly favors branch $1$ when $Q = 1$ and branch $0$ when $Q = 0$. 
The branch of the returned witness becomes correlated with $Q$, allowing an observer to infer the hidden bit with better-than-chance accuracy. 
Thus, relative to an unbiased policy independent of $Q$, the machine communicates additional information entirely through its legitimate computational choices, without appending a message or sacrificing task performance.

%% file: sections/proof_task_objective.tex
\section{Task Objective Proofs}
\label{app:task}

\subsection{Bounding the Task Objective}

We now demonstrate how to bound the task objective.
This is achieved through a sequence of steps that focus on different aspects of task performance.
For each round \(t\), define the good-estimation event
\begin{equation*}
\mathcal{E}_t
=
\left\{
    \left|\bar p_t-p_t\right|<\frac{1}{16d}
\right\}
\cap
\left\{
    \left|\hat q_t-p_t\right|<\frac{1}{8d}
\right\},
\end{equation*}
where $\bar p_t$ is the sample estimate obtained by $A$ prior to adding the perturbation.
Let $N$ denote the protocol stopping time.
Define the uniform good-estimation event
\[
    \mathcal{E}
    =   
    \bigcap_{t=1}^{N}\mathcal{E}_t \cap \{|\hat p_N^{\oracle}-p_N| <\frac{1}{8d}\}
\]
and the no-abort event by
\(
    \mathcal{K}
    =
    \{\text{the protocol does not abort}\}.    
\)
So that $\mathcal{E}$ is well-defined, we set $\hat{p}_T^\oracle = p_T$ if no abort occurs.
The good-estimation event states allows us to analyse the protocol under well-behaved sample estimates.
We first focus on the good-estimation event under small perturbations $\Delta$, which implies that the debate does not abort.
That is, if $\hat q_t$ and $\hat p_t^{\oracle}$ are ``good'' estimates, as defined above, then an abort under honest strategy $\sigma_B^{\honest}$ implies 
\begin{align}
    |\hat p_t - \hat q_t| 
    \geq 
    \frac{1}{2d} 
    \implies 
    |\hat p_t - p_t| 
    \geq 
    \frac{3}{8d} 
    \implies
    |\hat p_t^{\oracle} - \hat p_t| \geq \frac{1}{4d},
    \label{eqn:good-est-abort}
\end{align}
where the last statement contradicts $\mathcal{E}$ for perturbation $\Delta \le 1/(16d)$.
Therefore, an abort is not possible under event $\mathcal{E}$.
To bound the task objective on small perturbations, we decompose $V$'s acceptance probability as follows.
\begin{align*}
\Pr[V\to 1\mid \Delta]
&=
\Pr[V\to 1\mid\mathcal{E}, \Delta]\,\Pr[\mathcal{E}\mid \Delta]
+
\Pr[V\to 1\mid\neg\mathcal{E}, \Delta]\,\Pr[\neg\mathcal{E}\mid \Delta]
\end{align*}
Thus, the task objective satisfies,
\begin{align}
    \Pr[V\to 1\mid\mathcal{E}, \Delta] \Pr[\mathcal{E}]  
    \le \Pr[V\to 1, \Delta] 
    \le \Pr[V\to 1\mid\mathcal{E}, \Delta]
    +\Pr[\neg\mathcal{E}], \label{eq:task-objective-bound} 
\end{align}
where, for perturbation size $\Delta < 1/(16d)$, $\{\mathcal{E}, \Delta\} \subseteq \mathcal{K}$.

For larger perturbations, we focus on the outcome of an abort under good estimation.
By \eqref{eqn:good-est-abort}, an abort under good estimation implies that $\hat p_t$ exceeds the verifier's rejection threshold (which can occur for large enough $\Delta)$.
Thus, acceptance is impossible on \(\mathcal{E}\cap\neg\mathcal{K}\), and hence
\[
    \Pr\!\left[V\to 1,\mathcal{E},\neg\mathcal{K}\right]
    =
    0.
\]
We conclude that acceptance can occur only on \(\mathcal{E}\cap\mathcal{K}\) or on the failure event \(\neg\mathcal{E}\). 
Therefore,
\begin{align*}
    \Pr[V\to 1]
    &=
    \Pr[V\to 1,\mathcal{E},\mathcal{K}]
    +
    \Pr[V\to 1,\mathcal{E},\neg\mathcal{K}]
    +
    \Pr[V\to 1,\neg\mathcal{E}] \\
    &\le
    \Pr[\mathcal{E}\cap\mathcal{K}]
    +
    \Pr[\neg \mathcal{E}].
\end{align*}

If \(\Pr[\mathcal{E}]>0\), this can equivalently be written as
\begin{align}
    \Pr[V\to 1]
    &\le
    \Pr[\mathcal{E}\cap\mathcal{K}]
    +
    \Pr[\neg \mathcal{E}] \nonumber \\
    &=
    \Pr[\mathcal{E}]
    \Pr[\mathcal{K}\mid\mathcal{E}]
    +
    \Pr[\neg \mathcal{E}] \nonumber \\
    &=
    \Pr[ \mathcal{E}]
    \Pr[\mathcal{K}\mid\mathcal{E}]
    +
    \Pr[\neg \mathcal{E}] \nonumber\\
    &\le
    \Pr[\mathcal{K}\mid\mathcal{E}]
    +
    \Pr[\neg \mathcal{E}]. \label{eqn:task_objective_abort}
\end{align}

The following subsections will bound each term in \eqref{eq:task-objective-bound} and \eqref{eqn:task_objective_abort} to provide a characterization of the expected task objective as $\Delta$ grows.
We begin with the good-estimation event.

\subsection{Likelihood of good-estimation}
We now provide a bound for the good-estimation event.
We require the following version of the Chernoff bound.
\begin{lemma}\label{lem:chernoff}
    Let $X_1, \ldots, X_R$ be independent Bernoulli random variables, each taking value $1$ with probability $p$ and $0$ with probability $(1-p)$.
    Let $\bar \mu = \tfrac{1}{R}\sum_{i=1}^R X_i$ be the empirical mean of the random variables.
    Then
    \[
        \Pr\left[|\bar \mu - p| \geq s < 2\exp\left(\frac{-s^2R}{3} \right)\right]
    \]
\end{lemma}
\begin{lemma}\label{lem:good_est}
Suppose that
\(
    R\geq 192d^{2}\log(100T).
\)
Then every strategy $\sigma = \sigma_A \circ \sigma_B$ parameterized by \(\sigma_A \sim (R,\square)\) and \(\sigma_B \sim (R,\square)\) satisfies
    \[
        \Pr[\neg \mathcal{E}] \leq \frac{1}{300}
    \] 
\end{lemma}
\begin{proof}
If the estimates \(\bar{p}_{t}\), \(\hat g_t\), and \(\hat p_t^{\oracle}\) are the means of \(R\) independent Bernoulli samples with probability \(p_t\), Lemma~\ref{lem:chernoff} and a union bound gives
\[
    \Pr\!\left[
        \lvert \bar{p}_{t}-p_{t}\rvert
        \geq \frac{1}{16d}
        \land 
        \lvert \hat{q}_{t}-p_{t}\rvert
        \geq \frac{1}{8d}
    \right]
    \leq
    2\exp\!\left(-\frac{R}{192d^{2}}\right)+ 2\exp\!\left(-\frac{R}{96d^{2}}\right).
\]
Therefore, applying a union bound over all $T$ rounds yields
\begin{align*}
    \Pr[\neg \mathcal{E}]
    &\leq 4T\exp\!\left(-\frac{R}{192d^{2}}\right)
    \leq \frac{T}{600T}
    =\frac{1}{600},
\end{align*}
by the choice of $R$.
Lastly, for $V$ with $r=192d^2\log100$ samples,
\[
    \Pr\!\left[
        \lvert \hat{p}_{t}^{\oracle}-p_{t}\rvert
        \geq \frac{1}{8d}\right]
    \leq 2\exp\left( -\frac{r}{48d^2} \right)
    \leq \frac{4}{100^2}
    \leq \frac{1}{600}
\]
Thus, combining, we observe
\begin{align*}
    \Pr[\neg \mathcal{E}] 
    &\leq \Pr\!\left[
        \lvert \bar{p}_{t}-p_{t}\rvert
        \geq \frac{1}{4d}
        \land 
        \lvert \hat{q}_{t}-p_{t}\rvert
        \geq \frac{1}{4d}
    \right]
    + \Pr\!\left[
        \lvert \hat{p}_{t}^{\oracle}-p_{t}\rvert
        \geq \frac{1}{4d}\right] \\
    &\leq \frac{1}{600} + \frac{1}{600} = \frac{1}{300}.
\end{align*}
\end{proof}

\subsection{Stability of Task Success}

We now establish bounds on the task objective under the good estimation event.
We record a direct consequence of the Lipschitz property of the oracle machine.
\begin{fact}
\label{fact:lipschitz-oracle-stability}
Let \(M\) be \(K\)-Lipschitz oracle machine and denote
\(
    \Pr[M^{O}(x)=1] = \alpha^\star.
\)
Suppose that any non-aborting simulation of \(M\) which, at every history \(\rvh\), replaces the reference transition probability \(p(\rvh)\) by a possibly randomized, history-dependent probability \(\tilde p(\rvh)\) satisfying
\[
    |\tilde p(\rvh)-p(\rvh)| \le \frac{1}{d}
\]
has acceptance probability within \(K/d\) of \(\alpha^\star\).
\end{fact}

\begin{lemma}\label{lem:task-stability}
Let \(M\) be \(K\)-Lipschitz oracle machine and denote
\(
    \Pr[M^{O}(x)=1] = \alpha^\star.
\) 
Fix a strategy $\sigma_A$ with perturbation magnitude
\(
    \Delta \le \frac{3}{4d}.
\)
Then, on the protocol of Figure~\ref{fig:stochastic-oracle-debate},
\[
    \Pr[V\to 1]
    \le
    \alpha^\star+\frac{L}{d}+\frac{1}{300}.
\]
If, in addition, for value $\delta \ge 0$, \(\{\mathcal{E},\Delta=\delta\}\) implies that the protocol does not abort, then
\[
    \left|\Pr[V\to 1\mid \mathcal{E},\Delta =\delta]-\alpha^\star\right|
    \le
    \frac{L}{d}+\frac{1}{300}.
\]
\end{lemma}

\begin{proof}
Construct a simulation $\mathcal{S}$ of the protocol of Figure~\ref{fig:stochastic-oracle-debate} that follows the reporting strategy $\sigma_A$ but ignores all abort instructions and continues to the terminal machine output.
Extend $\sigma_A$ to $\mathcal{S}$ on histories that would occur only after an abort.

At each round, let \(\hat p_t\) be the report proposed by this simulation.
Instead of using \(\hat p_t\) directly, use its projection
\(\tilde p_t\) onto the interval
\[
    \left[
        \max\left\{0,p_t-\frac{1}{d}\right\},
        \min\left\{1,p_t+\frac{1}{d}\right\}
    \right].
\]
By construction,
\[
    |\tilde p_t-p_t|\le \frac{1}{d}
\]
at every history, including histories on which estimation fails. 
Consequently, if \(\tilde X\) denotes the terminal output of $\mathcal{S}$, Fact~\ref{fact:lipschitz-oracle-stability} gives
\begin{align}
    \left|\Pr[\tilde X=1]-\alpha^\star\right|
    \le
    \frac{K}{d}.\label{eqn:lipschitz}
\end{align}
Couple $\mathcal{S}$ and the actual protocol using the same random samples and transition randomness for as long as their executions coincide.
On \(\mathcal{E}\), the perturbation bound implies
\begin{align*}
    |\hat p_t-p_t|
    \le
    |\hat p_t-\bar p_t|
    +
    |\bar p_t-p_t| 
    <
    \frac{3}{4d}
    +
    \frac{1}{4d} 
    =
    \frac{1}{d}.
\end{align*}
Thus, the projection does not change any report along the actual execution on \(\mathcal{E}\). The two simulations therefore coincide until the protocol aborts or terminates.

On \(\mathcal{E}\), if the protocol accepts, it cannot have aborted, since, by~\eqref{eqn:good-est-abort}, every abort on \(\mathcal{E}\) results in rejection. 
Its entire computation consequently agrees with the auxiliary simulation, which also outputs \(1\).
Hence,
\[
    \{V\to 1\}\cap \mathcal{E}
    \subseteq
    \{\tilde X=1\}.
\]
Therefore,
\begin{align*}
    \Pr[V\to 1]
    &\le
    \Pr[V\to 1,\mathcal{E}]
    +
    \Pr[\neg\mathcal{E}] \\
    &\le
    \Pr[\tilde X=1]
    +
    \frac{1}{300} \\
    &\le
    \alpha^\star+\frac{K}{d}+\frac{1}{300}.
\end{align*}

If \(\{\mathcal{E}, \Delta\}\) also implies survival, then the two outputs agree throughout \(\mathcal{E}\).  
Then, for $\delta  \ge 0$, where  \(\{\mathcal{E}, \Delta =\delta\} \subseteq \mathcal{K}\),
\[
    \Pr[V\to 1\mid \mathcal{E}, \Delta=\delta ]
    =
    \Pr[\tilde X=1 \mid \mathcal{E}, \Delta=\delta],
\]
and hence
\[
    \left|\Pr[V\to 1\mid \mathcal{E}, \Delta=\delta]-\Pr[\tilde X=1]\right|
    \le
    \Pr[\neg\mathcal{E}]
\]
Combining this inequality with \eqref{eqn:lipschitz} gives
\[
    \left|\Pr[V\to 1\mid \mathcal{E}, \Delta=\delta]-\alpha^\star\right|
    \le
    \frac{L}{d}+\frac{1}{300},
\]
as claimed.
\end{proof}

\subsection{Non-abort survival}

We move to bounding the task objective for general perturbations.
By \eqref{eqn:task_objective_abort}, it suffices to provide an upper bound on \(\Pr[\mathcal{K}\mid \mathcal{E}]\).
We will require the following result.
Let \(W_1,\ldots,W_n\) be independent, mean-zero random variables with finite third absolute moments and $F_{\mathsf{std}}$ denote the cumulative distribution function of the Normal distribution. 
Writing
\[
    S_n=\sum_{i=1}^n W_i,
    \qquad
    \sigma^2=\sum_{i=1}^n \mathbb E[W_i^2]>0,
    \qquad
    \rho=\frac{\sum_{i=1}^n\mathbb E|W_i|^3}{\sigma^3},
\]
the Berry-Esseen theorem states the following. 
\begin{lemma}\label{thm:berry}
    For constant \(C_{\mathsf{BE}} = 0.56\),
\[
    \sup_{x\in\mathbb R}
    \left|
        \Pr[S_n\leq x]
        -
        F_{\mathsf{std}}\!\left(\frac{x}{\sigma}\right)
    \right|
    \leq C_{\mathsf{BE}} \cdot \rho.
\]
\end{lemma}
We now establish some useful notation.
Fix a perturbation magnitude \(\Delta \ge 0\).
All probabilities in this section refer to this fixed perturbation. 
Recall that the report at round \(t\) is
\[
    \hat p_t
    =
    \clip(\bar p_t + Q\Delta),
\]
and the cross-examiner aborts when
\[
    |\hat p_t-\hat q_t| \ge \tau = \frac{1}{2d}.
\]
Let \(\mathcal{K}_t\) denote the event that no cross-examination abort occurs during the first \(t\) rounds, with \(\mathcal{K}_0=\Omega\), and write
\(
    \mathcal{K}=\mathcal{K}_T.
\)
Define the estimation-noise difference
\(
    Z_t = \bar p_t-\hat q_t.
\)
In the absence of clipping, the discrepancy equals \(Q\Delta+Z_t\).
With clipping, however, this identity need not hold.
Throughout, assume that
\(
    \Pr[\mathcal{E}]>0.
\)

\subsubsection{One-Round Gaussian Approximation}

We first condition only on the past, rather than on the global event \(\mathcal{E}\).
Throughout we assume that $p_t \in (0.01,0.99)$.
Under this assumption, the event $\mathcal{E}$, and $\Delta<1/d$, no clipping occurs under the choice of $d=\lceil 150K \rceil$ and $K\ge 1$.
For any surviving history \(\vh\), define
\[
    r_t(\Delta,\vh)
    =
    \Pr\!\left[
        |\hat p_t-\hat q_t|<\tau
        \,\middle|\,
        \mathcal{K}_{t-1}, \rvh_{t-1}=\vh
    \right].
\]
We assume that, conditional on this history and either value of \(Q\) having positive conditional probability, the fresh samples
\[
    X_{t,1},\ldots,X_{t,R},
    Y_{t,1},\ldots,Y_{t,R}
\]
are mutually independent Bernoulli random variables with common mean
\(
    p_t=p_t(\vh),
\)
and that
\[
    \bar p_t
    =
    \frac{1}{R}\sum_{i=1}^R X_{t,i},
    \qquad
    \hat q_t
    =
    \frac{1}{R}\sum_{i=1}^R Y_{t,i}.
\]
In particular, at a fixed history, this fresh-sample law is the same for either value of \(Q\). 
Define the conditional probability of survival as
\[
    r_t^0(\Delta,\vh)
    =
    \Pr\!\left[
        \mathcal |Q\Delta+Z_t|<\tau
        \,\middle|\,
        \mathcal{K}_{t-1}, \rvh_{t-1}=\vh
    \right].
\]
For \(p_t\in(0.01,0.99)\), define
\(
    \sigma_t
    =
    \sqrt{{2p_t(1-p_t)}/{R}}
\)
and
\[
    g_t(\Delta)
    =
    \Pr[|\Delta + \sigma_t z| < \tau],
\]
where $z \sim \mathsf{N}(0,1)$.

\begin{lemma}\label{lem:unclipped-one-round-survival}
Assume $p_t\in(0.01,0.99)$ and that no clipping occurs for all configurations of machine input $M$, then
\[
    \left|
        r_t^0(\Delta,\vh)-g_t(\Delta)
    \right|
    \le
    1/300.
\]
\end{lemma}

\begin{proof}
Fix a surviving history \(\vh\) and a value of \(Q\) with positive conditional
probability. Under this conditioning,
\[
    Z_t
    =
    \frac{1}{R}\sum_{i=1}^R (X_{t,i}-p_t)
    -
    \frac{1}{R}\sum_{i=1}^R (Y_{t,i}-p_t).
\]
The summands are independent and centered, and
\[
    \operatorname{Var}(Z_t)
    =
    \frac{2p_t(1-p_t)}{R}
    =
    \sigma_t^2.
\]
For a Bernoulli random variable \(X\) with mean \(p_t\),
\begin{align*}
    \mathbb{E}[|X-p_t|^3]
    &=
    p_t(1-p_t)^3 + (1-p_t)p_t^3 \\
    &=
    p_t(1-p_t)\bigl[p_t^2+(1-p_t)^2\bigr].
\end{align*}
Therefore, the standardized sum of third absolute moments is
\begin{align*}
    \rho_{t,R}
    &=
    \frac{
        2R\cdot R^{-3}\cdot
        p_t(1-p_t)\bigl[p_t^2+(1-p_t)^2\bigr]
    }{
        \sigma_t^3
    } \\
    &=
    \frac{1}{\sqrt{2R}}\,
    \frac{p_t^2+(1-p_t)^2}{\sqrt{p_t(1-p_t)}}.
\end{align*}
Lemma~\ref{thm:berry} consequently gives, uniformly over \(x\),
\[
    \left|
        \Pr\!\left[
            Z_t\le x
            \,\middle|\,
            \mathcal{K}_{t-1},\rvh_{t-1}=\vh,Q
        \right]
        -
        F_{\mathsf{std}}\!\left(\frac{x}{\sigma_t}\right)
    \right|
    \le
    \frac{C_{\mathsf{BE}}}{\sqrt{2R}}\,
    \frac{p_t^2+(1-p_t)^2}{\sqrt{p_t(1-p_t)}} \leq \frac{1}{600},
\]
by the choice of $R$ and assumption on $p_t$.

For \(Q=+1\), the unclipped survival event is
\[
    -\tau-\Delta < Z_t < \tau-\Delta.
\]
Applying the distribution-function bound at the interval endpoints yields an error of at most \(2/600\) relative to \(g_t(\Delta)\). 
The conditional law of \(Z_t\) is symmetric about zero, since it is the difference of two independent, identically distributed sample means. 
Thus the same interval probability and bound hold for \(Q=-1\). 
Averaging over the conditional distribution of \(Q\) proves the result.
\end{proof}

Now we can provide the final conditional survival bound.
\begin{lemma}\label{lem:conditional-survival-bound}
Fix a perturbation \(\Delta\) and a round \(t\in\{1,\ldots,T\}\).
Then
\[
    \Pr[\mathcal{K}\mid\mathcal{E}]
    \le
    \min\left\{
        1,\,
        \mathbb{E}\!\left[
        g_t(\Delta)
        \,\middle|\,
        \mathcal{K}_{t-1}
    \right]+\frac{1}{100}
    \right\}.
\]
\end{lemma}

\begin{proof}
By Lemma~\ref{lem:unclipped-one-round-survival}, for every surviving history \(\vh\), the one-round survival probability satisfies
\[
    \Pr\!\left[
        \mathcal{K}_t
        \,\middle|\,
        \mathcal{K}_{t-1}
    \right]
    \le
    g_t(\Delta)+\frac{1}{300}.
\]
Define
\(
    G_t(\Delta)
    =
    \mathbb{E}\!\left[
        g_t(\Delta)
        \,\middle|\,
        \mathcal{K}_{t-1}
    \right].
\)
Averaging the one-round bound over histories conditional on \(\mathcal{K}_{t-1}\) gives
\begin{align*}
    \Pr[\mathcal{K}_t\mid \mathcal{K}_{t-1}]
    &=
    \mathbb{E}\!\left[
        \Pr\!\left[
            \mathcal{K}_t
            \,\middle|\,
            \mathcal{K}_{t-1},\rvh_{t-1}
        \right]
        \,\middle|\,
        \mathcal{K}_{t-1}
    \right] \\
    &\le
    \mathbb{E}\!\left[
        g_t(\Delta)+\frac{1}{300}
        \,\middle|\,
        \mathcal{K}_{t-1}
    \right] \\
    &=
    G_t(\Delta)+\frac{1}{300}.
\end{align*}
Since \(\mathcal{K}\subseteq \mathcal{K}_t\subseteq \mathcal{K}_{t-1}\), it follows that
\begin{align*}
    \Pr[\mathcal{K}]
    &\le
    \Pr[\mathcal{K}_t] \\
    &=
    \Pr[\mathcal{K}_{t-1}]
    \Pr[\mathcal{K}_t\mid\mathcal{K}_{t-1}] \\
    &\le
    G_t(\Delta)+\frac{1}{300}.
\end{align*}
The law of total probability yields
\begin{align*}
    \Pr[\mathcal{K}\mid\mathcal{E}] - \Pr[\mathcal{K}]
    &=
    \Pr[\neg \mathcal{E}]
    \left(
        \Pr[\mathcal{K}\mid\mathcal{E}]
        -
        \Pr[\mathcal{K}\mid\neg \mathcal{E}]
    \right) \\
    &\le
    \Pr[\neg\mathcal{E}] \\
    &\le
    \frac{1}{150},
\end{align*}
By Lemma~\ref{lem:good_est}.
Consequently,
\[
    \Pr[\mathcal{K}\mid\mathcal{E}]
    \le
    G_t(\Delta)+\frac{1}{300}+\frac{1}{150} = G_t(\Delta) + \frac{1}{100}.
\]
Combining this with the trivial bound
\(
    \Pr[\mathcal{K}\mid\mathcal{E}]\le 1
\)
proves the claim.
\end{proof}

\subsection{Proof Statements}

Theorem~\ref{thm:task_objective} follows from \eqref{eq:task-objective-bound} with the application of Lemmas~\ref{lem:good_est} and~\ref{lem:task-stability} and the observations that $K/d\le 1/150$ and $\{\mathcal{E}, \Delta < 1/(16d)\} \subseteq \mathcal{K}$.

Theorem~\ref{thm:task_general_bound} follows from \eqref{eqn:task_objective_abort} and the application of Lemmas~\ref{lem:good_est} and~\ref{lem:conditional-survival-bound} for $t=1$.

%% file: sections/proof_latent_objective.tex
\section{Proof of Theorem~\ref{thm:full-disclosure}}
\label{app:latent}
We begin with setting notation.
Fix \(X=x\) and let
\(
    \rvy=(\ervy_1,\ldots,\ervy_T)
\)
be the random variable representing the transcript of $M(x)$ and 
\(
    \va = (\eva_1, \ldots, \eva_T)
\)
denote the realized transcript.
Unless stated otherwise, all probabilities, expectations, and information quantities are conditional on \(X=x\).  
Write $h_x=H(Q\mid X=x)$.
The view of $V$ is defined as
\[
    \Pi_V(\rvh(X,\sigma))=
    \begin{cases}
        (\eva_1,\ldots,\eva_T), & \text{ on event } \mathcal{K},\\
        (\eva_1,\ldots,\eva_t,\hat p_t, \{\oracle(\va_{I(t)})\}),
        & \text{ on event }\neg \mathcal{K}_t \mid \mathcal{K}_{t-1}.
    \end{cases}
\]
These transcript types are distinguishable, so the abort indicator and, upon abort, its round are recoverable from
\(\Pi_V\).
Notably, $V$ does not observe the internal samples of $A$ or $B$ directly.
For some \(0<\beta<1/2\) we assume that $\Pr[Q=1 \mid X=x]\in[\beta,1-\beta]$.
Thus, there is initial uncertainty in the hidden hypothesis.

\subsection{Belief updating and confidence truncation}

Define the hidden hypothesis posterior as 
\[
    \mu_t=
    \Pr(Q=1\mid \rvh_{t-1},\mathcal{K}_{t-1}).
\]
These are exact Bayesian posteriors for the observed non-abort channel. 
At each entering history \(\rvh_{t-1}\), the expected posterior divergence satisfies the identity
\begin{equation}
    \mathbb E\left[
        D_{\mathsf{KL}}\bigl(
            \Bern(\mu_t)
            \Vert\Bern(\mu_{t-1})
        \bigr)
        \,\middle|\,\rvh_{t-1}, \mathcal{K}_{t-1}
    \right]
    =
    I(Q;\ervy_t\mid \rvh_{t-1}, \mathcal{K}_{t-1}).
    \label{eq:posterior-update-identity}
\end{equation}
Thus, the expected information gained in an update cannot exceed the remaining uncertainty. 
More generally, any valid upper bound on the channel's mutual information also bounds the expected posterior divergence in \eqref{eq:posterior-update-identity}.
This identity follows from Bayes' rule and is consistent with sequential hypothesis testing, where posterior log-odds accumulate observation log-likelihood ratios.

In our latent disclosure accounting, expected $\mathsf{KL}$ increments accumulate according to the mutual-information chain rule and their total is bounded by the initial entropy $h_x$.
We assume that all information used in updating is included in the observed history.

\subsection{Single round disclosure}
We begin by bounding the amount of information disclosed in a single round of debate.
We look at the instances of abort and non-abort separately.
For probability laws \(P\) and \(Q\), with divergence measured in bits, Pinsker's inequality states that
\begin{align}\label{eqn:pinsker}
    D_{\mathsf{KL}}(P\Vert Q)
    \geq
    \frac{2}{\ln 2}\|P-Q\|_{\mathsf{TV}}^2.
\end{align}

\begin{lemma}
\label{lem:single-bit-disclosure}
For $\mu_t \in (\beta, 1-\beta)$, 
\begin{equation*}
    I(Q;\ervy_t\mid \rvh_{t-1},\mathcal{K}_{t-1})
    \geq
    \frac{2}{\ln 2}\mu_t(1-\mu_t)\Delta^2
    \geq
    \frac{2\beta(1-\beta)}{\ln 2}\Delta^2.
    \label{eq:single-bit-information}
\end{equation*}
\end{lemma}

\begin{proof}
Define the hypothesis conditional sampling probabilities as
\[
    p_{\pm} = \Pr[\ervy_t \mid Q \pm 1, \rvh_{t-1},\mathcal{K}_{t-1}]
\]
and the mixture 
\(
     p=\mu_t p_+ +(1-\mu_t)p_-.
\)
Since a mixture of Bernoulli laws is Bernoulli, the conditional mutual information is
\begin{align*}
    I(Q;\ervy_t\mid \rvh_{t-1},\mathcal{K}_{t-1})
    &=
    \mu_t D_{\mathsf{KL}}\left(
        \Bern(p_+)\Vert\Bern(\bar p)
    \right)+
    (1-\mu_t)D_{\mathsf{KL}}\left(
        \Bern(p_-)\Vert\Bern(\bar p)
    \right).
\end{align*}
For Bernoulli laws, total variation distance equals the absolute difference of their parameters. 
Thus, \eqref{eqn:pinsker} gives
\begin{align}
    I(Q;\ervy_t\mid \rvh_{t-1},\mathcal{K}_{t-1})
    &\geq \frac{2}{\ln 2}
        \left[
            \mu_t(p_+-\bar p)^2
            +(1-\mu_t)(p_--\bar p)^2
        \right] \nonumber\\
    &=\frac{2}{\ln 2}
        \mu_t(1-\mu_t)(p_+-p_-)^2. \label{eqn:single_lower}
\end{align}

Let \(\nu\) denote the distribution of \(\bar p_t\).
As the distribution of \(\bar p_t\) is independent of \(Q\),
\[
    p_+-p_-
    =
    \int_{[0,1]}
    \left[
        \clip(s+\Delta)
        -\clip(s-\Delta)
    \right],\nu(ds).
\]
For every \(s\in[0,1]\) and \(0<\Delta\leq\tfrac12\),
\[
    \clip(s+\Delta)
    -\clip(s-\Delta)
    \geq\Delta.
\]
Hence \(p_+-p_-\geq\Delta\). 
Substitution into \eqref{eqn:single_lower}, followed by \(\mu(1-\mu)\geq\beta(1-\beta)\), proves the stated result.
Averaging over the entering histories gives the conditional mutual-information bound.
\end{proof}

Disclosure during an abort is given by the following.
\begin{lemma}
\label{lem:abort-full-disclosure}
Suppose that \(2R\Delta\notin\mathbb Z\).
Then, the report $\hat p_t$ fully determines \(Q\). 
In other words, 
\[
     H(Q\mid \hat p _t)=0.
\]
\end{lemma}

\begin{proof}
Recall that the public report upon abort is
\[
    \hat p_t=\frac{K_t}{R}+Q\Delta,
    \qquad K_t\in\{0,\ldots,R\},
\]
Under the two hypotheses, the report belongs respectively to the grids
\[
    \mathcal G_+
    =\left\{\frac{k}{R}+\Delta:0\leq k\leq R\right\},
    \qquad
    \mathcal G_-
    =\left\{\frac{k}{R}-\Delta:0\leq k\leq R\right\}.
\]
If these grids intersect, there exist integers \(k,\ell\) such that
\[
    \frac{k}{R}+\Delta
    =
    \frac{\ell}{R}-\Delta,
\]
which implies
\[
    2R\Delta=\ell-k\in\mathbb Z.
\]
The assumed nonintegrality therefore makes the grids disjoint.
Membership in the observed grid determines \(Q\).
It follows that \(H(Q\mid \hat p _t)=0\). 
\end{proof}

\subsection{Full debate disclosure}

We are now in a position to proceed with the main proof of information disclosure.
Let $h_2$ denote the binary entropy function.
\begin{theorem}
\label{thm:abort-partitioned-disclosure}
For $\beta\in(0,\1/2)$ and $\mu_0 \in (\beta, 1-\beta)$, 
let
\(
    k_\beta=\frac{2\beta(1-\beta)}{\ln 2}.
\)
Then 
\begin{equation*}
    I(Q;\Pi_V(\rvh)\mid X=x)
    \geq
    \frac{h_x -h_2(\beta)}{2} \min \left\{1, \frac{k_\beta \Delta^2 T}{1-h_2(\beta)}\right\}.
\end{equation*}
\end{theorem}

\begin{proof}
Throughout the proof, condition on \(X=x\) and measure entropy in bits. 
Stop observing the protocol at the first of the following events:
\begin{enumerate}
    \item an abort;
    \item the posterior leaves \((\beta,1-\beta)\);
    \item completion of round \(T\).
\end{enumerate}
Denote the resulting stopped transcript by \(\rvu \subseteq \rvy \). 
On history $\rvh$, since \(\rvu\) is determined by $\Pi_V(\rvh)$, the data-processing inequality gives
\(
    I(Q;\Pi_V(\rvh) \mid X=x)
    \ge
    I(Q;\rvu \mid X=x).
\)
We bound the latter quantity.

Fix $t$ and an active pre-round history \(\rvh_{t-1}\), and let \(\mathcal{A}_t=\neg\mathcal{K}_t \mid \mathcal{K}_{t-1}\) indicate an abort in the current round.
Let \(\ervu_t\) be the single bit appearing at round $t$.

Let \(N \le T\) be the number of rounds observed in \(\rvu\).
Padding the transcript after stopping with a fixed symbol, the conditional chain rule gives
\begin{align*}
    I(Q;\rvu )
    &=
    \sum_{t=1}^T I(Q; \ervu_t | \rvh_{t-1}, \mathcal{K}_{t-1}) \\
    &=
    \sum_{t=1}^T I(Q; \ervy_t | \rvh_{t-1}, \mathcal{K}_{t-1}) \cdot \Pr[N\ge t] \\
    &\ge
    k_\beta \Delta^2 \sum_{t=1}^T \Pr(N \ge t) \\
    &=
    k_\beta \Delta^2\,\mathbb{E}[N],
\end{align*}
where the inequality follows from Lemma~\ref{lem:single-bit-disclosure}.
Let \(\mathcal U\) be the event that, after \(T\) rounds, the protocol has neither aborted nor reached the confidence threshold, and set
\(
    q = \Pr(\mathcal U).
\)
On \(\mathcal U\), all \(T\) rounds are observed. 
Hence
\(
    \mathbb{E}[N] \ge Tq,
\)
and therefore
\begin{align}
     I(Q;\rvu \mid \mathcal{U}) \ge k_{\beta}\Delta^2Tq.
     \label{eqn:MI-1}
\end{align}

Outside \(\mathcal U\), by Lemma~\ref{lem:abort-full-disclosure}, on abort posterior entropy is zero by full disclosure and on reaching the confidence threshold it is at most
\(
    h_2(\beta).
\)
Averaging over stopped transcripts gives
\begin{align*}
    H(Q \mid \rvu)
    &\le
    h_2(\beta)(1-q)+q \\
    &=
    h_2(\beta)+(1-h_2(\beta))q.
\end{align*}
Since
\[
     I(Q;\rvu ) = h_x - H(Q \mid \rvu),
\]
we obtain
\begin{align}
    I(Q;\rvu ) \ge h_x - h_2(\beta) - (1-h_2(\beta))q.
    \label{eqn:MI-second}
\end{align}

Bringing everything together, set
\(
    D = h_x-h_2(\beta),
\)
and
\(
    c = 1-h_2(\beta).
\)
The assumption \(\mu_0\in(\beta,1-\beta)\) ensures that \(D>0\). 
Inequalities \eqref{eqn:MI-1} and \eqref{eqn:MI-second} imply that
\begin{align*}
     I(Q;\rvu) \ge \max \left\{ k_{\beta}\Delta^2 Tq, D-cq \right\}
\end{align*}
Multiplying the first inequality by \(c\) and the second by \(k_{\beta}\Delta^2T\), then adding,
eliminates \(q\).
In other words,
\begin{align*}
    cI(Q;\rvu ) + k_{\beta}\Delta^2T I(Q;\rvu )
    &\ge
    ck_{\beta}\Delta^2Tq + k_{\beta}\Delta^2T(D-cq) \\
    &=
    k_{\beta}\Delta^2TD.
\end{align*}
Thus,
\begin{align*}
    (c+k_{\beta}\Delta^2T)I(Q;\rvu ) 
    &\ge k_{\beta}\Delta^2TD \\
    \therefore
    I(Q;\rvu ) &\ge D \cdot \frac{k_{\beta}\Delta^2}{c+k_{\beta}\Delta^2}
\end{align*}
Finally, for every \(u\ge 0\),
\[
    \frac{u}{1+u}
    \ge
    \frac{1}{2}\min\{1,u\}.
\]
Taking
\[
    u = \frac{k_\beta \Delta^2 T}{c}
    =
    \frac{k_\beta \Delta^2 T}{1-h_2(\beta)}
\]
and applying data processing yields
\[
    I(Q;\Pi_V(\rvh) \mid X=x)
    \ge
    \frac{h_x-h_2(\beta)}{2}
    \min\left\{
        1,\,
        \frac{k_\beta \Delta^2 T}{1-h_2(\beta)}
    \right\}.
\]
\end{proof}

In conclusion, Theorem~\ref{thm:full-disclosure} is a corollary of Theorem~\ref{thm:abort-partitioned-disclosure} for a fixed choice of $\beta$.

%% file: sections/proof_mitigation.tex
\section{Proof of Theorem~\ref{thm:hhd_drift_oracle}}

\begin{figure}[t]
\centering
\begin{minipage}{0.96\linewidth}
\small

\begin{center}
    \textbf{Debate protocol with error and drift aborts}
\end{center}

\footnotesize
\noindent
All parties have access to \(\oracle\), an input \(x\in\{0,1\}^n\), and a \(K\)-Lipschitz probabilistic oracle machine \(M\).
\(A\) claims that
\(
    \Pr\left[M(x)=1\right]\geq \frac{2}{3},
\)
and the cross-examiner \(B\) disputes this claim.
Let \(W\in [T]\) be the trailing-window length used by an abort drift challenge.

\begin{enumerate}[leftmargin=*, itemsep=2pt, topsep=2pt]

\item
Let
\(
    \rvy=(\ervy_1,\ldots,\ervy_T)
\)
be the random variable representing the transcript of $M(x)$ and set
\(
    d=\lceil 150K\rceil.
\)

\item
For debate round $t\in[T]$:

\begin{enumerate}

\item[(a)-(d)] Identical to the equivalent steps in Figure~\ref{fig:stochastic-oracle-debate}.

\item[(e)]
$B$ can $\textsf{abort error}$ if (it believes) $A$ has deviated on the most recent update or $\textsf{abort drift}$ if (it believes) $A$ has deviated cumulatively over the most recent $W$ updates and $t\ge W$.

\end{enumerate}

\item
If \textsf{abort error} was issued in round $t$:

\begin{enumerate}

\item
$V$ draws
\(
    192d^2\log 100
\)
independent samples of $\ervy_t$, conditioned on
\(
    \ervy_{I(t)}=\va_{I(t)},
\)
and uses their sample mean to obtain an estimate
$\hat p_t^{\oracle}$ of
\(
    \Pr\left[
        \ervy_t=1
        \middle|
        \rvy_{I(t)}=\va_{I(t)}
    \right]
\)

\item
$V$ checks whether
\(
    \left|
        \hat p_t^{\oracle}-\hat p_t
    \right|
    \geq
    \frac{1}{4d}.
\)
If the inequality holds, $V$ outputs $0$; otherwise, $V$ outputs $1$.

\end{enumerate}

\item
If \textsf{abort drift} was issued in round $t$:

\begin{enumerate}

\item
For each $s \in J= \{t-W+1, \ldots, t\}$,
$V$ draws
\(
    192d^2\log (100W)
\)
independent samples of $\ervy_s$, conditioned on
\(
    \ervy_{I(s)}=\va_{I(s)},
\)
and uses the sample mean to obtain an estimate $\hat p_s^{\oracle}$ of
\(
    \Pr\left[
        \ervy_s=1
        \middle|
        \rvy_{I(s)}=\va_{I(s)}
    \right]
\)

\item
$V$ checks whether
\(
    \left|\sum_{s \in J}
    \left(
        \hat p_s^{\oracle}-\hat p_s
    \right)\right|
    \geq
    \frac{\sqrt{N}}{4d}.
\)
If the inequality holds, $V$ outputs $0$; otherwise, $V$ outputs $1$.

\end{enumerate}

\item
If the protocol was not aborted, then the verifier outputs $a_T$.

\end{enumerate}

\end{minipage}
\caption{
Debate protocol for a stochastic oracle with error and drift aborts.
}
\label{fig:error-drift-stochastic-oracle-debate}
\end{figure}

\subsection{Concentration Bounds}

We require the conditional form of Hoeffding's Lemma.
\begin{lemma}[Hoeffding's Lemma]
\label{lem:conditional-hoeffding}
Let \(\mathcal{G}\) be a sigma-algebra, and let \(Y\) be a random variable satisfying
\(
    a \le Y \le b
\)
almost surely, where \(a\) and \(b\) are deterministic constants. 
If
\(
    \mathbb{E}[Y \mid \mathcal{G}] = 0,
\)
then, for every real \(\lambda\),
\[
    \mathbb{E}\left[
        \exp(\lambda Y)
        \middle|
        \mathcal{G}
    \right]
    \le
    \exp\left(
        \frac{\lambda^2(b-a)^2}{8}
    \right)
\]
almost surely.
\end{lemma}

The following bounds will be used to bound ``good'' sampling behavior during protocol execution.
\begin{lemma}
\label{lem:conditional-concentration-adaptive-means}
Let \((\mathcal{F}_t)_{t\ge 0}\) be a filtration. 
For each round \(t\), let \(p_t\in[0,1]\) be \(\mathcal{F}_{t-1}\)-measurable. 
Conditional on \(\mathcal{F}_{t-1}\), suppose
\(
    X_{t,1},\ldots,X_{t,R}
\)
are independent Bernoulli random variables with common mean \(p_t\). 
Define the \(\mathcal{F}_t\)-measurable quantities
\(
    \hat p_t
    =
    R^{-1}\sum_{i=1}^R X_{t,i},
\)
and
\(
    D_t
    =
    \hat p_t-p_t.
\)
Then, almost surely, the following bounds hold.
\begin{enumerate}
    \item For every real \(\lambda\),
    \(
        \mathbb{E}\left[
            \exp(\lambda D_t)
            \middle|
            \mathcal{F}_{t-1}
        \right]
        \le
        \exp\left(\frac{\lambda^2}{8R}\right).
    \)

    \item For every \(u>0\),
    \(
        \Pr\left[
            |D_t|\ge u
            \middle|
            \mathcal{F}_{t-1}
        \right]
        \le
        2\exp(-2Ru^2).
    \)

    \item For any deterministic integers
    \(0\le a<b\) and every \(u>0\),
    \[
        \Pr\left[
            \left|
                \sum_{s=a+1}^b D_s
            \right|
            \ge u
            \middle|
            \mathcal{F}_a
        \right]
        \le
        2\exp\left(
            -\frac{2Ru^2}{b-a}
        \right).
    \]
\end{enumerate}
\end{lemma}

\begin{proof}
Conditional on \(\mathcal{F}_{t-1}\), each variable
\(
    \frac{X_{t,i}-p_t}{R}
\)
has mean zero and lies in the interval
\(
    \left[-{p_t}/{R}, {1-p_t}/{R}\right],
\)
whose length is \(1/R\). 
By Lemma~\ref{lem:conditional-hoeffding},
\[
    \mathbb{E}\left[
        \exp\left(
            \lambda \frac{X_{t,i}-p_t}{R}
        \right)
        \middle|
        \mathcal{F}_{t-1}
    \right]
    \le
    \exp\left(\frac{\lambda^2}{8R^2}\right).
\]
Conditional independence of the \(R\) samples therefore yields
\begin{align}
    \mathbb{E}\left[
        \exp(\lambda D_t)
        \middle|
        \mathcal{F}_{t-1}
    \right]
    &=
    \prod_{i=1}^R
    \mathbb{E}\left[
        \exp\left(
            \lambda \frac{X_{t,i}-p_t}{R}
        \right)
        \middle|
        \mathcal{F}_{t-1}
    \right]\nonumber \\
    &\le
    \exp\left(\frac{\lambda^2}{8R}\right). \label{eqn:first-assertion}
\end{align}
This proves the first assertion.
For \(\lambda>0\), the Markov inequality  gives
\[
    \Pr\left[
        D_t \ge u
        \middle|
        \mathcal{F}_{t-1}
    \right]
    \le
    \exp\left(
        -\lambda u + \frac{\lambda^2}{8R}
    \right).
\]
Choosing \(\lambda = 4Ru\) gives the upper-tail bound
\[
    \Pr\left[
        D_t \ge u
        \middle|
        \mathcal{F}_{t-1}
    \right]
    \le
    \exp(-2Ru^2).
\]
Applying the same argument to \(-D_t\) and adding the two tail bounds proves the second assertion.

For the third assertion, write
\[
    S_{a,b}
    =
    \sum_{s=a+1}^b D_s .
\]
Since \(S_{a,b-1}\) is \(\mathcal{F}_{b-1}\)-measurable, \eqref{eqn:first-assertion} implies
\[
\begin{aligned}
    \mathbb{E}\left[
        \exp(\lambda S_{a,b})
        \middle|
        \mathcal{F}_a
    \right]
    &=
    \mathbb{E}\left[
        \exp(\lambda S_{a,b-1})
        \mathbb{E}\left[
            \exp(\lambda D_b)
            \middle|
            \mathcal{F}_{b-1}
        \right]
        \middle|
        \mathcal{F}_a
    \right] \\
    &\le
    \exp\left(\frac{\lambda^2}{8R}\right)
    \mathbb{E}\left[
        \exp(\lambda S_{a,b-1})
        \middle|
        \mathcal{F}_a
    \right].
\end{aligned}
\]
Iterating over the \(b-a\) rounds gives
\[
    \mathbb{E}\left[
        \exp(\lambda S_{a,b})
        \middle|
        \mathcal{F}_a
    \right]
    \le
    \exp\left(
        \frac{(b-a)\lambda^2}{8R}
    \right).
\]

The Markov inequality at
\(
    \lambda = \frac{4Ru}{b-a},
\)
gives
\[
    \Pr\left[
        S_{a,b} \ge u
        \middle|
        \mathcal{F}_a
    \right]
    \le
    \exp\left(
        -\frac{2Ru^2}{b-a}
    \right).
\]
Applying the same argument to \(-S_{a,b}\) and adding the two tails completes
the proof.
\end{proof}

\subsection{Protocol Completeness}

Given the honest behaviour by both parties,
we assume that the transcript update is a fresh Bernoulli draw with parameter \(\hat p_t\), and that \(p_t = \Pr[\ervy_t \mid \rvy_{I(t) = \va_{I(s)}}\) is the corresponding true transition probability of \(M\). 
The protocol uses full windows
\(
    J_t = \{t-W+1,\ldots,t\},
\)
with drift challenges permitted only for \(t\ge W\).
The verifier's drift statistic must sum the \(s\)-indexed discrepancies. 
That that, for probabilistic Turing machine $M$, the standard bounded-error guarantee is
\(
    \Pr[M^\oracle(x)=1]\ge \frac{2}{3}.
\)

\begin{lemma}[Task correctness]\label{lem:task-correctness}
Let $M$ be a $K$-Lipschitz probabilistic Turing machine that decides $L$
\(
    R \ge 192d^2 \log(100T),
\)
For any \(x\in L\), any fixed window length \(W\in [T]\), honest strategy \(\sigma_A\sim (R,0)\), and arbitrary cross-examiner strategy \(\sigma_B\),
\[
    \Pr[V \to 1] \ge \frac{3}{5}.
\]
\end{lemma}

\begin{proof}

Simulate an auxiliary execution $\mathcal{S}$ that generates all \(T\) transcript updates using the reporting strategy $\sigma_A$, ignoring any challenge for purposes of continuing the transcript. 
Let $\rvy^{\mathsf{s}}$ denote the random transcript of the simulated execution and before each round \(t\), let \(\rvh_{t-1}^{\mathsf{s}}\) contain the complete history on the simulated execution. 
Conditional on \(\rvh_{t-1}^{\mathsf{s}}\), the honest reporter draws \(R\) fresh, independent Bernoulli samples
\(
    X_{t,1},\ldots,X_{t,R},
\)
each with mean \(p_t\), and reports
\[
    \hat p_t^{\mathsf{s}}
    =
    \frac{1}{R}\sum_{j=1}^R X_{t,j}.
\]
The next transcript bit \(\ervy_t^{\mathsf{s}}\) is generated using independent randomness with success probability \(\hat p_t^{\mathsf{s}}\). 
Consequently,
\begin{align*}
    \Pr[\ervy_t^{\mathsf{s}}=1\mid \rvh_{t-1}^{\mathsf{s}}]
    &=
    \mathbb{E}[\hat p_t^{\mathsf{s}}\mid \rvh_{t-1}^{\mathsf{s}}] 
    =
    p_t.
\end{align*}
Thus, the simulated transcript has exactly the transition probabilities of the true execution of
\(M^\oracle(x)\). 
If \(\ervy_T^{\mathsf{s}}\) denotes its terminal output, then
\[
    \Pr[\ervy_T^{\mathsf{s}}=1]
    =
    \Pr[M^\oracle(x)=1]
    \ge
    \frac{2}{3}.
\]

Couple the actual protocol to this simulated execution using the same samples and transcript-update randomness until the actual protocol stops. 
If no challenge occurs, the actual output equals \(\ervy_T^\simulate = \eva_T^\simulate\).
The simulated continuation after a challenge is used only for analysis and is
generated independently of the subsequent audit randomness of $V$.

Define the reporting error
\(
    D_t^\simulate = \hat p_t^\simulate-p_t.
\)
For each round \(t\), let \(\mathcal{E}_t^{(1)}\) be the event
\(
    |D_t^\simulate| < \frac{1}{8d}.
\)
For each \(t\ge W\), let \(\mathcal{E}_t^{(2)}\) be the event
\[
    \left|
        \sum_{s\in J_t} D_s^\simulate
    \right|
    <
    \frac{\sqrt{W}}{8d}.
\]
Let \(\mathcal{G}^\simulate\) be the intersection of all these individual and window events in the simulated execution.

By Lemma~\ref{lem:conditional-concentration-adaptive-means}, for every round \(t\), 
\begin{align}
    \Pr[\neg \mathcal{E}_t^{(1)} \mid \rvh_{t-1}^\simulate] 
    =
    \Pr\left[
        |D_t^\simulate|\ge \frac{1}{8d}
        \middle|
        \rvh_{t-1}^\simulate
    \right]
    \le
    2\exp\left(-\frac{R}{32d^2}\right).
    \nonumber
\end{align}
Similarly, for window errors, Lemma~\ref{lem:conditional-concentration-adaptive-means} gives, for every $u >0$,
\[
    \Pr\left[
        \left|\sum_{s\in J_t}D_s^\simulate\right|\ge u \mid \rvh_{t-W}^\simulate
    \right]
    \le
    2\exp\left(-\frac{2Ru^2}{W}\right).
\]
Therefore, it follows that
\[
    \Pr[\neg \mathcal{E}_t^{(2)} \mid \rvh_{t-W}^\simulate ]
    \le
    2\exp\left(-\frac{R}{32d^2}\right).
\]
There are \(T\) individual events and at most \(T\) window events. Therefore, by
a union bound,
\begin{align}
    \Pr[\neg \mathcal{G}^\simulate]
    &\le
    4T\exp\left(-\frac{R}{32d^2}\right) \nonumber\\
    &\le
    4T(100T)^{-6}.\label{eqn:not_G}
\end{align}

Let $\mathcal{A}_t^{\mathsf{(e)}}$ denote the event that  \textsf{abort error} occurs at time $t$.
Condition on a complete pre-audit history $\rvh^\simulate$ at which $\mathcal{A}_t^{\mathsf{(e)}}$ occurs.
This fixes \(p_t\) and \(\hat p_t\).
Suppose 
\(
    \mathcal{G}^\simulate
\)
occurs.
If $V$ rejects, then
\[
    |\hat p_t^\oracle-\hat p_t|
    \ge
    \frac{1}{4d},
\]
so the triangle inequality implies
\[
    |\hat p_t^\oracle-p_t|
    >
    \frac{1}{8d}.
\]
$V$ uses
\(
    r^{(\mathsf{e})}
    =
    \left\lceil 192d^2\log 100 \right\rceil
\)
fresh, independent samples. 
Therefore, by Lemma~\ref{lem:conditional-concentration-adaptive-means},
\begin{align}
    \Pr[V \to 0\mid \rvh^\simulate, \mathcal{G}^\simulate, \mathcal{A}_t^{\mathsf{(e)}}]
    &\le
    2\exp\left(-\frac{r^{(\mathsf{e})}}{32d^2}\right)
    \le
    2\cdot 100^{-6},
    \label{eqn:verifier-abort}
\end{align}
under the protocol execution couple to the simulated execution.

Let $\mathcal{A}_t^{\mathsf{(d)}}$ denote the event that  \textsf{abort drift} occurs at time $t\ge W$.
As above, now condition on a complete pre-audit history $\rvh^\simulate$ at which $\mathcal{A}_t^{\mathsf{(d)}}$ occurs.
Suppose $\mathcal{G}^\simulate$ occurs.
For each \(s\in J_t\), $V$ draws
\(
    r^{(\mathsf{d})}
    =
    \left\lceil 192d^2\log(100W)\right\rceil
\)
fresh samples, independently across all audited rounds and sample indices.
Define
\[
    S_V
    =
    \sum_{s\in J_t}(\hat p_s^\oracle-p_s).
\]
Conditioned on $\rvh^\simulate$, \(S_V\) is a sum of \(Wr^{(\mathsf{d})}\) independent centered variables, each having range length \(1/r^{(\mathsf{d})}\). 
Hence, by Lemma~\ref{lem:conditional-concentration-adaptive-means}
\[
    \Pr\left[
        |S_V|\ge u
        \middle|
        \rvh^\simulate, \mathcal{G}^\simulate, \mathcal{A}_t^{\mathsf{(d)}}
    \right]
    \le
    2\exp\left(-\frac{2r^{(\mathsf{d})}u^2}{W}\right).
\]
If
\(
    |S_V| < {\sqrt{W}}/{8d},
\)
the triangle inequality gives
\[
    \left|
        \sum_{s\in J_t}(\hat p_s^\oracle-\hat p_s)
    \right|
    <
    \frac{\sqrt{W}}{4d}.
\]
Therefore, $V$ accepts the challenge audit. 
Consequently,
\begin{align}
    \Pr[V \to 0 \mid \rvh^\simulate, \mathcal{G}^\simulate, \mathcal{A}_t^{\mathsf{(d)}}]
    \le
    2\exp\left(-\frac{r^{(\mathsf{d})}}{32d^2}\right) 
    \le
    2(100W)^{-6} 
    \le
    2\cdot 100^{-6}.\label{eqn:v-reject-drift}
\end{align}

Let \(\mathcal{J}\) be the event that a challenge occurs and the verifier outputs \(0\).
By \eqref{eqn:verifier-abort} and \eqref{eqn:v-reject-drift},
\begin{align}
    \Pr[\mathcal{J}\cap \mathcal{G}] 
    &\le
    \sum_{t,\mathsf{c}} \Pr[\mathcal{A}_t^{(\mathsf{c})} \cap \mathcal{G}^\simulate] \Pr[V \to 0 \mid \mathcal{A}_t^{(\mathsf{c})} \cap \mathcal{G}^\simulate] \nonumber \\
    &\le  2\cdot 100^{-6}  \sum_{t,\mathsf{c}} \Pr[\mathcal{A}_t^{(\mathsf{c})} \cap \mathcal{G}^\simulate] \nonumber\\
    &\le 
    2\cdot 100^{-6}. \label{eqn:eventJcapG}
\end{align}
Conditional on each pre-audit history, the audit randomness is independent of the auxiliary continuation. 
Consequently, further conditioning on $\mathcal{G}^\simulate$ does not change the conditional law of the audit samples.

Under the coupling, if \(\ervy_T^\simulate=1\) and the actual protocol outputs \(0\), then a challenge must have occurred and been rejected. 
Indeed, without a challenge, the actual output equals \(\ervy_T^\simulate=1\). 
Thus, by \eqref{eqn:not_G} and \eqref{eqn:eventJcapG},
\begin{align*}
    \Pr[V\to 1]
    &\ge
    \Pr[\ervy_T^\simulate=1]-\Pr[\mathcal{J}] \\
    &\ge
    \Pr[\ervy_T^\simulate=1]-\Pr[\neg \mathcal{G}^\simulate]-\Pr[\mathcal{J}\cap \mathcal{G}^\simulate] \\
    &\ge
    \frac{2}{3} - 4T(100T)^{-6} - 2\cdot 100^{-6} > \frac{3}{5}
\end{align*}
The argument holds for every legal cross-examiner strategy \(\sigma_B\), proving the lemma.
\end{proof}

\subsection{Smaller Admissible Region}

For the second part of Theorem~\ref{thm:hhd_drift_oracle} we require the following.
\begin{lemma}
\label{lem:poor_strategy}
Let \(L\) be any language decided by a Turing machine \(M\) and \(W\in[T]\) be the sliding-window parameter.
For 
\(
    R \ge 192 d^2 \log(100T),
\)
\(
    \Delta \ge 1/(d\sqrt{W}),
\)
\(
    \tau_1 = 1/(2d),
\)
and
\(
    \tau_2 = \sqrt{W}/(2d),
\)
let $\sigma_A \sim (R, \Delta), $ and $\sigma_B \sim (R, \tau_1, \tau_2)$. 
\[
    \Phi_A(x, \sigma_A \circ \sigma_B) = \Pr[V\gets 1] <\frac{1}{150}.
\]
\end{lemma}

\begin{proof}[Revised proof]
Fix \(Q \in \{-1,1\}\). 
All probabilities below are conditional on this fixed value of \(Q\).
Construct an simulated execution $\mathcal{S}$ that continues for \(T\) rounds according to strategies $\sigma_A$ and $\sigma_B$, ignoring challenges.
Couple the actual protocol to this simulated execution until its first challenge. 
The simulated continuation uses randomness independent of the $V$'s fresh audit randomness.

At round \(t\), let \(p_t\) be the true oracle probability at the current state $\rvh^\simulate$.
Let \(\bar p_t\) and \(\hat q_t\) be \(A\)'s and \(B\)'s unperturbed sample means, each based on \(R\) fresh \(\Bern(p_t)\) samples. 
Define
\(
    D_A(t)^\simulate = \bar p_t^\simulate - p_t,
\)
\(
    D_B(t)^\simulate = \hat q_t^\simulate - p_t,
\)
\(
    S_A(t) = \sum_{s\in J_t} D_A(s)^\simulate,
\)
and
\(
    S_B(t) = \sum_{s\in J_t} D_B(s)^\simulate.
\)
Note \(S_A(t)\) is defined using the unperturbed estimate \(\bar p_s^\simulate\), not the perturbed report.

Let \(\mathcal{G}_{AB}\) be the event that, for both \(i\in\{A,B\}\),
\(
    |D_i(t)^\simulate| < \frac{1}{8d}
    \text{for every } ,
\)
for every $t\in\{1,\ldots,T\}$, and
\(
    |S_i(t)| < {\sqrt{W}}/{8d}
\)
for every $t\in\{W,\ldots,T\}$

Conditional on the pre-round history $\rvh_{t-1}^\simulate$ and \(Q\), each sample mean is unbiased
for \(p_t\).
By Lemma\ref{lem:conditional-concentration-adaptive-means}, for each \(i\in\{A,B\}\),
\[
    \Pr\left[
        |D_i(t)^\simulate| \ge \frac{1}{8d}
        \middle|
        Q, \rvh_{t-1}^\simulate 
    \right]
    \le
    2\exp\left(-\frac{R}{32d^2}\right),
\]
and
\[
    \Pr\left[
        |S_i(t)| \ge \frac{\sqrt{W}}{8d}
        \middle|
        Q, \rvh_{t-W}^\simulate
    \right]
    \le
    2\exp\left(-\frac{R}{32d^2}\right).
\]
There are at most \(4T\) events in the definition of \(\mathcal{G}_{AB}\).
Consequently,
\begin{align}
    \Pr[\mathcal{G}_{AB}^{c}\mid Q]
    &\le
    8T\exp\left(-\frac{R}{32d^2}\right) \nonumber  \\
    &\le
    8T(100T)^{-6}. \label{eqn:tag-1}
\end{align}

On \(\mathcal{G}_{AB}\), the assumed absence of clipping gives
\(
    \hat p_t^\simulate = \bar p_t^\simulate + Q\Delta .
\)
Under the coupled execution.
if an \textsf{abort error} challenge has not already occurred, the execution reaches the first full window \(J_W\). 
At that window, for $\Delta \ge 1/(d\sqrt{W})$,
\begin{align}
    \left|
        \sum_{s\in J_W}(\hat p_s-\hat q_s)
    \right|
    &=
    \left|
        QW\Delta + S_A(W)-S_B(W)
    \right| \nonumber\\
    &\ge
    W\Delta - |S_A(W)| - |S_B(W)| \nonumber \\
    &>
    \frac{\sqrt{W}}{d}
    -
    \frac{2\sqrt{W}}{8d} \nonumber\\
    &=
    \frac{3\sqrt{W}}{4d} \nonumber\\
    &>
    \frac{\sqrt{W}}{2d}. \label{eqn:challenge-occurs}
\end{align}
Thus \(B\) issues a challenge no later than round \(W\).
It remains to show that either possible challenge type leads to output \(0\), except with uniformly small audit probability.

Let $\mathcal{A}_t^{\mathsf{(e)}}$ denote the event that  \textsf{abort error} occurs at time $t$.
Condition on a complete pre-audit history $\rvh^\simulate$ at which $\mathcal{A}_t^{\mathsf{(e)}}$ occurs.
This fixes $\bar p_t^\simulate, \hat q_t^\simulate$, and $p_t$.
By the challenge rule,
\[
    |\hat p_t^\simulate-\hat q_t^\simulate| > \frac{1}{2d}.
\]
On events \(\mathcal{G}_{AB}\) and \(|\hat q_t-p_t|<1/(8d)\) the following holds
\begin{align*}
     |\hat p_t^\simulate-p_t|
    \ge
    |\hat p_t^\simulate-\hat q_t^\simulate|
    -
    |\hat q_t^\simulate-p_t|  
    >
    \frac{1}{2d}-\frac{1}{8d} 
    =
    \frac{3}{8d}.
\end{align*}
If the $V$'s fresh estimate satisfies
\(
    |\hat p_t^{\oracle}-p_t| < {1}/{(8d)},
\)
then
\begin{align*}
    |\hat p_t^{\oracle}-\hat p_t^\simulate|
    &\ge
    |\hat p_t^\simulate-p_t|
    -
    |\hat p_t^{\oracle}-p_t| \\
    &>
    \frac{3}{8d}-\frac{1}{8d} \\
    &=
    \frac{1}{4d}.
\end{align*}
Hence, \(V\) outputs \(0\). 
It follows that
\begin{align}
    \Pr[V \to 1\mid \mathcal{A}_t^{\mathsf{(e)}}, \rvh^\simulate, \mathcal{G}_{AB}]
    \le  \Pr[|\hat p_t^{\oracle}-p_t| > {1}/{(8d)}] \nonumber \\
    \ge 
    2\exp\left(-\frac{r^{\mathsf{(e)}}}{32d^2}\right)
    \le
    2\cdot 100^{-6}. \label{eqn:tag-2}
\end{align}

Similarly, let $\mathcal{A}_t^{\mathsf{(d)}}$ denote the event that  \textsf{abort drift} occurs at time $t$.
Condition on a complete pre-audit history $\rvh^\simulate$ at which $\mathcal{A}_t^{\mathsf{(d)}}$ occurs.
Define
\[
    S_V(t)
    =
    \sum_{s\in J_t}
    \left(\hat p_s^{\oracle}-p_s\right).
\]
On \(\mathcal{G}_{AB}\) and the event $|S_V(t)| < \frac{\sqrt{W}}{8d}$,
\begin{align*}
    \left|
        \sum_{s\in J_t}
        (\hat p_s-\hat p_s^{\oracle})
    \right|
    &=
    \left|
        QW\Delta + S_A(t)-S_V(t)
    \right| \\
    &\ge
    W\Delta - |S_A(t)| - |S_V(t)| \\
    &>
    \frac{\sqrt{W}}{d}
    -
    \frac{\sqrt{W}}{8d}
    -
    \frac{\sqrt{W}}{8d} \\
    &=
    \frac{3\sqrt{W}}{4d} \\
    &>
    \frac{\sqrt{W}}{4d}.
\end{align*}
Therefore \(V\) again outputs \(0\) under good sample estimates.
Thus, by Lemma~\ref{lem:conditional-concentration-adaptive-means},
\begin{align}
    \Pr[V \to 1\mid \mathcal{A}_t^{\mathsf{(d)}}, \rvh^\simulate, \mathcal{G}_{AB}]
    &\le 
    \Pr\left[
        |S_V(t)| \ge \frac{\sqrt{W}}{8d}
        \middle|
        \rvh^\simulate, Q
    \right] \nonumber\\
    &\le
    2\exp\left(-\frac{r^{(\mathsf{d})}}{32d^2}\right)
    \le
    2(100W)^{-6}
    \le
    2\cdot 100^{-6}.\label{eqn:tag-3}
\end{align}

By \eqref{eqn:challenge-occurs}, on \(\mathcal{G}_{AB}\), a challenge necessarily occurs at a time $t\le W$. 
Its time and type select one audit.
That is, they do not create multiple audit-failure penalties. 
By \eqref{eqn:tag-2} and \eqref{eqn:tag-3},
\begin{align*}
    \Pr[V \to 1 , \mathcal G_{AB} | Q]
    &= 
    \sum_{t=1}^W \sum_{\mathsf{c}} \Pr[{V \to 1} \cap \mathcal{A}_{t}^{(\mathsf{c})} \cap \mathcal{G}_{AB} | Q] \\
    &= 
    \sum_{t=1}^W \sum_{\mathsf{c}} \Pr[{V \to 1} \mid \mathcal{A}_{t}^{(\mathsf{c})} \cap \mathcal{G}_{AB}  Q] \cdot \Pr[\mathcal{A}_{t}^{(\mathsf{c})} \cap \mathcal{G}_{AB} \mid Q] \\
    &\le
    2\cdot 100^{-6} \sum_{t=1}^W \sum_{\mathsf{c}} \Pr[{V \to 1} \Pr[\mathcal{A}_{t}^{(\mathsf{c})} \cap \mathcal{G}_{AB} \mid Q]  \\
    &\le
     2\cdot 100^{-6}.
\end{align*}
Together with \eqref{eqn:tag-1},
\begin{align*}
    \Pr[V\to 1\mid Q]
    &\le
    \Pr[\neg \mathcal{G}_{AB}\mid Q]
    +
    \Pr[V\to 1 , \mathcal{G}_{AB}\mid Q] \\
    &\le
    8T(100T)^{-6}
    +
    2\cdot 100^{-6} \\
    &<
    \frac{1}{150}.
\end{align*}
The bound is uniform in \(Q\).
Averaging over \(Q\) completes the proof.
\end{proof}